\def\year{2021}\relax
\documentclass[letterpaper]{article}
\usepackage{aaai21} 
\usepackage{times} 
\usepackage{helvet}
\usepackage{courier} 
\usepackage[hyphens]{url} 
\usepackage{graphicx}
\usepackage{natbib} 
\usepackage{caption}
\usepackage{amsfonts}
\usepackage{amsmath}
\usepackage{amsthm}
\usepackage{bm}
\usepackage[english]{babel}
\usepackage{soul,color}

\newtheorem{theorem}{Theorem}[section]

\newtheorem{lemma}[theorem]{Lemma}
\newtheorem{definition}[theorem]{Definition}

\usepackage{algorithm}
\usepackage[noend]{algpseudocode}

\usepackage{subcaption}

\newcounter{alphasect}
\def\alphainsection{0}

\let\oldsection=\section
\def\section{
  \ifnum\alphainsection=1
    \addtocounter{alphasect}{1}
  \fi
\oldsection}

\renewcommand\thesection{\ifnum\alphainsection=1\Alph{alphasect}\else\arabic{section}\fi}

\newenvironment{alphasection}{
  \ifnum\alphainsection=1
    \errhelp={Let other blocks end at the beginning of the next block.}
    \errmessage{Nested Alpha section not allowed}
  \fi
  \setcounter{alphasect}{0}
  \def\alphainsection{1}
}{
  \setcounter{alphasect}{0}
  \def\alphainsection{0}
}

\usepackage[switch]{lineno}
\usepackage{multirow}

\title{Exploration by Maximizing R\'enyi Entropy for Reward-Free RL Framework}
\author {
    Chuheng Zhang\protect\thanks{These authors contributed equally to this work.} \quad
    Yuanying Cai\footnotemark[1] \quad
    Longbo Huang \quad
    Jian Li \\
}
\affiliations {
    Institute for Interdisciplinary Information Sciences (IIIS), Tsinghua University \\
    zhangchuheng123@live.com, cai-yy16@mails.tsinghua.edu.cn, \\ longbohuang@tsinghua.edu.cn,
    lijian83@mail.tsinghua.edu.cn
}

\begin{document}
\maketitle

\begin{abstract}

Exploration is essential for reinforcement learning (RL).
To face the challenges of exploration, we consider a reward-free RL framework that completely separates exploration from exploitation and brings new challenges for exploration algorithms.
In the exploration phase, the agent learns an exploratory policy by interacting with a reward-free environment
and collects a dataset of transitions by executing the policy.
In the planning phase, the agent computes a good policy for any reward function based on the dataset without further interacting with the environment.
This framework is suitable for the meta RL setting where there are many reward functions of interest.
In the exploration phase, we propose to maximize the R\'enyi entropy over the state-action space 
and justify this objective theoretically.
The success of using R\'enyi entropy as the objective results from its encouragement to explore the hard-to-reach state-actions.
We further deduce a policy gradient formulation for this objective and design a practical exploration algorithm that can deal with complex environments.
In the planning phase, we solve for good policies given arbitrary reward functions using a batch RL algorithm.
Empirically, we show that our exploration algorithm is effective and sample efficient, and results in superior policies for arbitrary reward functions in the planning phase.

\end{abstract}

\section{Introduction}

The trade-off between exploration and exploitation is at the core of reinforcement learning (RL). 
Designing efficient exploration algorithm, while being a highly nontrivial
task, is essential to the success in many RL tasks \cite{burda2018exploration,ecoffet2019go}.
Hence, it is natural to ask the following high-level question: \emph{What can we achieve by pure exploration?}
To address this question, several settings related to meta reinforcement learning (meta RL) have been proposed 
(see e.g., \citealt{wang2016learning,duan2016rl,finn2017model}). 
One common setting in meta RL is to learn a model in a reward-free environment in the meta-training phase, and use the learned model as the initialization
to fast adapt for new tasks in the meta-testing phase \cite{eysenbach2018diversity,gupta2018unsupervised,nagabandi2018learning}.
Since the agent still needs to explore the environment under the new tasks in the meta-testing phase (sometimes it may need more new samples in some new task, and sometimes less), it is less clear how to evaluate the effectiveness of the exploration in the meta-training phase.
Another setting is to learn a policy in a reward-free environment and test the policy under the task with a specific reward function (such as the score in Montezuma's Revenge) without further training with the task \cite{burda2018exploration,ecoffet2019go,burda2018large}.
However, there is no guarantee that the algorithm has fully explored the transition dynamics of the environment
unless we test the learned model for arbitrary reward functions.
Recently, \citet{jin2020reward} proposed \emph{reward-free RL framework} that fully decouples exploration and exploitation.
Further, they designed a provably efficient algorithm that conducts a finite number of steps of reward-free exploration and returns near-optimal policies for arbitrary reward functions.
However, their algorithm is designed for the tabular case and can hardly be extended to continuous or high-dimensional state spaces since they construct a policy for each state.

The reward-free RL framework is as follows:
First, a set of policies are trained to explore the dynamics of the reward-free environment in the exploration phase (i.e., the meta-training phase). 
Then, a dataset of trajectories is collected by executing the learned exploratory policies. 
In the planning phase (i.e., the meta-testing phase), an arbitrary reward function is specified and a batch RL algorithm \cite{lange2012batch,fujimoto2018off} is applied to solve for a good policy solely based on the dataset, without further interaction with the environment.
This framework is suitable to the scenarios when there are many reward functions of interest or the reward is designed offline to elicit desired behavior.
The key to success in this framework is to obtain a dataset with good coverage over all possible situations in the environment with as few samples as possible, which in turn requires the exploratory policy to fully explore the environment.

Several methods that encourage various forms of exploration have been developed
in the reinforcement learning literature. 
The maximum entropy framework \cite{haarnoja2017reinforcement} maximizes the cumulative reward in the meanwhile maximizing the entropy over \emph{the action space} conditioned on each state.
This framework results in several efficient and robust algorithms, such as soft Q-learning \cite{schulman2017equivalence, nachum2017bridging}, SAC \cite{haarnoja2018soft} and MPO
\cite{abdolmaleki2018maximum}. 
On the other hand, maximizing \emph{the state space} coverage may results in better exploration. 
Various kinds of objectives/regularizations are used for better exploration
of the state space, including 
information-theoretic metrics \cite{houthooft2016vime,mohamed2015variational,eysenbach2018diversity} 
(especially the entropy of the state space, e.g.,  \citet{hazan2018provably,islam2019entropy}),
the prediction error of a dynamical model \cite{burda2018large,pathak2017curiosity,de2018curiosity}, 
the state visitation count \cite{burda2018exploration,bellemare2016unifying,ostrovski2017count} or 
others heuristic signals such as novelty \cite{lehman2011novelty,conti2018improving}, surprise \cite{achiam2017surprise} or curiosity \cite{schmidhuber1991possibility}.

In practice, it is desirable to obtain a single exploratory policy instead of a set of policies whose cardinality may be as large as that of the state space \citep[e.g.,][]{jin2020reward,misra2019kinematic}.
Our exploration algorithm outputs a single exploratory policy for the reward-free RL framework by maximizing \emph{the R\'enyi entropy over the state-action space} in the exploration phase. 
In particular, we demonstrate the advantage of using
state-action space, instead of the state space via a very simple example
(see Section \ref{sec:objective} and Figure \ref{fig:five-state}). 
Moreover, R\'enyi entropy generalizes a family of entropies, including the commonly used Shannon entropy.
We justify the use of R\'enyi entropy as the objective function theoretically by providing an upper bound on the number of samples in the dataset to ensure that a near-optimal policy is obtained for any reward function in the planning phase.

Further, we derive a gradient ascent update rule for maximizing the R\'enyi entropy over the state-action space. 
The derived update rule is similar to the vanilla policy gradient update with the reward replaced by a function of the discounted stationary state-action distribution of the current policy.
We use variational autoencoder (VAE) \cite{kingma2013auto} as the density model to estimate the distribution.
The corresponding reward changes over iterations which makes it hard to accurately estimate a value function under the current reward.
To address this issue, we propose to estimate a state value function using the off-policy data with the reward relabeled by the current density model.
This enables us to efficiently update the policy in a stable way using PPO \cite{schulman2017proximal}.
Afterwards, we collect a dataset by executing the learned policy.
In the planning phase, when a reward function is specified, we augment the dataset with the rewards and use a batch RL algorithm, batch constrained deep Q-learning (BCQ) \cite{fujimoto2018off,fujimoto2019benchmarking}, to plan for a good policy under the reward function. 
We conduct experiments on 
several
environments with discrete, continuous or high-dimensional state spaces.
The experiment results indicate that our algorithm is effective, sample efficient and robust in the exploration phase, and achieves good performance under the reward-free RL framework.

Our contributions are summarized as follows:

\begin{itemize}
    \item (Section \ref{sec:objective}) We consider the reward-free RL framework that separates exploration from exploitation and therefore places a higher requirement for an exploration algorithm. 
    To efficiently explore under this framework, we propose a novel objective that maximizes the R\'enyi entropy over the state-action space in the exploration phase and justify this objective theoretically.
    \item (Section \ref{sec:method}) We propose a practical algorithm based on a derived policy gradient formulation to maximize the R\'enyi entropy over the state-action space for the reward-free RL framework.
    \item (Section \ref{sec:experiment}) We conduct a wide range of experiments and the results indicate that our algorithm is efficient and robust in the exploration phase and results in superior performance in the downstream planning phase for arbitrary reward functions.
\end{itemize}

\section{Preliminary}
\label{sec:prelim}

A reward-free environment can be formulated as a controlled Markov process (CMP)
\footnote{
Although some literature use CMP to refer to MDP, we use CMP to denote a Markov process 
\cite{silver2015lecture} 
equipped with a control structure in this paper.
}
$(\mathcal{S}, \mathcal{A}, \mathbb{P}, \mu, \gamma)$ which specifies the state space $\mathcal{S}$ with $S=|\mathcal{S}|$, the action space $\mathcal{A}$ with $A=|\mathcal{A}|$, the transition dynamics $\mathbb{P}$, 
the initial state distribution $\mu\in\Delta^{\mathcal{S}}$ and the discount factor $\gamma$. 
A (stationary) policy $\pi_\theta:\mathcal{S} \to \Delta^\mathcal{A}$ parameterized by $\theta$ specifies the probability of choosing the actions on each state. 
The stationary discounted state visitation distribution (or simply the state distribution) under the policy $\pi$ is defined as $d_\mu^\pi(s):=(1-\gamma)\sum_{t=0}^\infty \gamma^t \text{Pr}(s_t=s|s_0 \sim \mu; \pi)$. 
The stationary discounted state-action visitation distribution (or simply the state-action distribution) under the policy $\pi$ is defined as $d_\mu^\pi(s, a):=(1-\gamma)\sum_{t=0}^\infty \gamma^t \text{Pr}(s_t=s, a_t=a|s_0 \sim \mu; \pi)$. Unless otherwise stated, we use $d_\mu^\pi$ to denote the state-action distribution. 
We also use $d_{s_0,a_0}^\pi$ to denote the distribution started from the state-action pair $(s_0, a_0)$ instead of the initial state distribution $\mu$. 

When a reward function $r: \mathcal{S}\times\mathcal{A} \to \mathbb{R}$ is specified, CMP becomes the Markov decision process (MDP) \cite{sutton2018introduction} $(\mathcal{S}, \mathcal{A}, \mathbb{P}, r, \mu, \gamma)$. 
The objective for MDP is to find a policy $\pi_\theta$ that maximizes the expected cumulative reward 
$J(\theta;r) := \mathbb{E}_{s_0 \sim \mu}[V^{\pi_\theta}(s_0;r)]$,
where 
$V^\pi(s;r)=\mathbb{E}[\sum_{t=0}^\infty \gamma^t r(s_t, a_t)|s_0=s; \pi]$ 
is the state value function.
We indicate the dependence on $r$ to emphasize that there may be more than one reward function of interest.
The policy gradient for this objective is $\nabla_\theta J(\theta; r) = \mathbb{E}_{(s,a)\sim d^{\pi_\theta}_\mu} [\nabla_\theta \log \pi_\theta(a|s) Q^{\pi_\theta} (s,a;r)]$ \cite{williams1992simple}, 
where
$Q^\pi(s,a;r)=\mathbb{E}[\sum_{t=0}^\infty \gamma^t r(s_t, a_t)|s_0=s, a_0=a; \pi] = \frac{1}{1-\gamma}\langle d_{s,a}^\pi, r \rangle$ 
is the Q function.

R\'enyi entropy for a distribution $d\in\Delta^\mathcal{X}$ is defined as $H_\alpha(d) := \frac{1}{1-\alpha} \log\left( \sum_{x\in\mathcal{X}} d^\alpha(x) \right)$, where $d(x)$ is the probability mass or the probability density function on $x$
(and the summation becomes integration in the latter case).
When $\alpha\to 0$, R\'enyi entropy becomes Hartley entropy and equals the logarithm of the size of the support of $d$.
When $\alpha\to 1$, R\'enyi entropy becomes Shannon entropy $H_1(d):=-\sum_{x\in\mathcal{X}} d(x) \log d(x)$
\cite{bromiley2004shannon,sanei2016renyi}.



Given a distribution $d\in\Delta^\mathcal{X}$, the corresponding density model $P_\phi: \mathcal{X} \to \mathbb{R}$ parameterized by $\phi$ gives the probability density estimation of $d$ based on the samples drawn from $d$. 
Variational auto-encoder (VAE) \cite{kingma2013auto} is a popular density model which maximizes the variational lower bound (ELBO) of the log-likelihood. 
Specifically, VAE maximizes the lower bound of $\mathbb{E}_{x\sim d}[\log P_\phi (x)]$, i.e.,  $\max_\phi \mathbb{E}_{x\sim d, z\sim q_\phi(\cdot|x)}[\log p_\phi (x|z)] - \mathbb{E}_{x\sim d} [D_\text{KL}(q_\phi(\cdot|x)||p(z))]$, where $q_\phi$ and $p_\phi$ are the decoder and the encoder respectively and $p(z)$ is a prior distribution for the latent variable $z$.

\section{The objective for the exploration phase}
\label{sec:objective}

The objective for the exploration phase under the reward-free RL framework is to induce an \emph{informative} and \emph{compact} dataset: 
The informative condition is that the dataset should have good coverage such that the planning phase generates good policies for arbitrary reward functions.
The compact condition is that the size of the dataset should be as small as possible to ensure a successful planning. 
In this section, we show that the R\'enyi entropy over the state-action space (i.e., $H_\alpha(d_\mu^\pi)$) is a good objective function for the exploration phase.
We first demonstrate the advantage of maximizing the state-action space entropy over maximizing the state space entropy with a toy example.
Then, we provide a motivation to use R\'enyi entropy by analyzing a deterministic setting.
At last, we provide an upper bound on the number of samples needed in the dataset for a successful planning if we have access to a policy that maximizes the R\'enyi entropy over the state-action space.
For ease of analysis, we assume the state-action space is discrete in this section and derive a practical algorithm that deals with continuous state-action space in the next section.

\begin{figure}[tbp]
   \centering
   \includegraphics[width=\columnwidth]{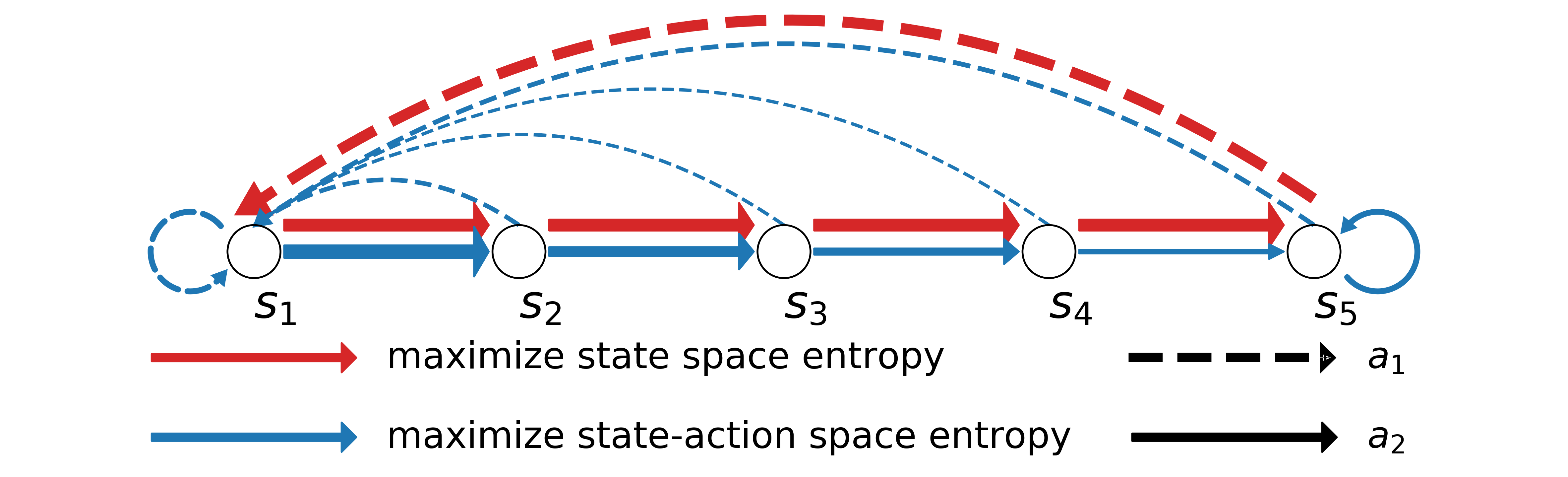}    
   \caption{
   A deterministic five-state MDP with two actions. With discount factor $\gamma=1$, a policy that maximizes the entropy of the discounted \emph{state} visitation distribution does not visit all the transitions while a policy that maximizes the entropy of the discounted \emph{state-action} visitation distribution visits all the transitions. 
   Covering all the transitions is important for the reward-free RL framework.
   The width of the arrows indicates the visitation frequency under different policies.}
   \label{fig:five-state}
\end{figure}

\textbf{Why maximize the entropy over the state-action space?}
We demonstrate the advantage of maximizing the entropy over \emph{the state-action space} with a toy example shown in Figure \ref{fig:five-state}. 
The example contains an MDP with two actions and five states.
The first action always drives the agent back to the first state while the second action moves the agent to the next state. 
For simplicity of presentation, we consider a case with a discount factor $\gamma=1$, but other values are similar. 
The policy that maximizes the entropy of the state distribution is a deterministic policy that takes the actions shown in red. 
The dataset obtained by executing this policy is poor since the planning algorithm fails when, in the planning phase, a sparse reward is assigned to one of the state-action pairs that it visits with zero probability (e.g., a reward function $r(s,a)$ that equals to $1$ on $(s_2, a_1)$ and equals to $0$ otherwise).
In contrast, a policy that maximizes the entropy of the state-action distribution avoids the problem. 
For example, 
executing the policy that maximizes the R\'enyi entropy with $\alpha=0.5$ over the state-action space, 
the expected size of the induced dataset is only 44 such that the dataset contains all the state-action pairs (cf. Appendix \ref{app:toy}).
Note that, when the transition dynamics is known to be deterministic, a dataset containing all the state-action pairs is sufficient for the planning algorithm to obtain an optimal policy since the full transition dynamics is known.

\textbf{Why using R\'enyi entropy?}
We analyze a deterministic setting where the transition dynamics is known to be deterministic.
In this setting, the objective for the framework can be expressed as a specific objective function for the exploration phase.
This objective function is hard to optimize but motivates us to use R\'enyi entropy as the surrogate.
%

We define $n:=SA$ as the cardinality of the state-action space.
Given an exploratory policy $\pi$, we assume the dataset is collected in a way such that the transitions in the dataset can be treated as i.i.d. samples from $d_\mu^\pi$, where $d_\mu^\pi$ is the state-action distribution induced by the policy $\pi$.

In the deterministic setting, we can recover the exact transition dynamics of the environment using a dataset of transitions that contains all the $n$ state-action pairs.
Such a dataset leads to a successful planning for arbitrary reward functions, and therefore satisfies the \emph{informative} condition. 
In order to obtain such a dataset that is also \emph{compact},
we stop collecting samples if the dataset contains all the $n$ pairs. 
Given the distribution $d_\mu^\pi = (d_1, \cdots, d_n)\in\Delta^n$ from which we collect samples, the average size of the dataset is $G(d_\mu^\pi)$, where
\begin{equation}
\label{eq:coupon}
G(d_\mu^\pi) = \int_0^\infty \left(1- \prod_{i=1}^n (1-e^{-d_i t})\right) dt,
\end{equation}
which is a result of the coupon collector’s problem \cite{flajolet1992birthday}. 
Accordingly, the objective for the exploration phase can be expressed as $\min_{\pi\in \Pi} G(d_\mu^\pi)$, where $\Pi$ is the set of all the feasible policies.
(Notice that due to the limitation of the transition dynamics, the feasible state-action distributions form only a subset of $\Delta^n$.)
We show the contour of this function on $\Delta^n$ in Figure \ref{fig:contour}a.
We can see that, when any component of the distribution $d_\mu^\pi$ approaches zeros,
$G(d_\mu^\pi)$ increases rapidly forming a \emph{barrier} indicated by the black arrow.
This barrier prevents the agent to visit a state-action with a vanishing probability, and thus encourages the agent to explore the hard-to-reach state-actions.

However, this function involves an improper integral which is hard to handle, and therefore cannot be directly used as an objective function in the exploration phase.
One common choice for a tractable objective function is Shannon entropy, i.e., $\max_{\pi\in\Pi} H_1(d_\mu^\pi)$
\cite{hazan2018provably,islam2019entropy}, 
whose contour is shown in Figure \ref{fig:contour}b.
We can see that, Shannon entropy may still be large when some component of $d_\mu^\pi$ approaches zero 
(e.g., the black arrow indicates a case where the entropy is relatively large but $d_1\to 0$).
Therefore, maximizing Shannon entropy may result in a policy that visits some state-action pair with a vanishing probability.
We find that the contour of R\'enyi entropy (shown in Figure \ref{fig:contour}c/d) aligns better with $G(d_\mu^\pi)$ and alleviates this problem.
In general, the policy $\pi$ that maximizes R\'enyi entropy may correspond to a smaller $G(d_\mu^\pi)$ than that resulted from maximizing Shannon entropy
for different CMPs
(more evidence of which can be found in Appendix \ref{app:CMP}).

\begin{figure}[t]
   \centering
   \includegraphics[width=0.8\columnwidth]{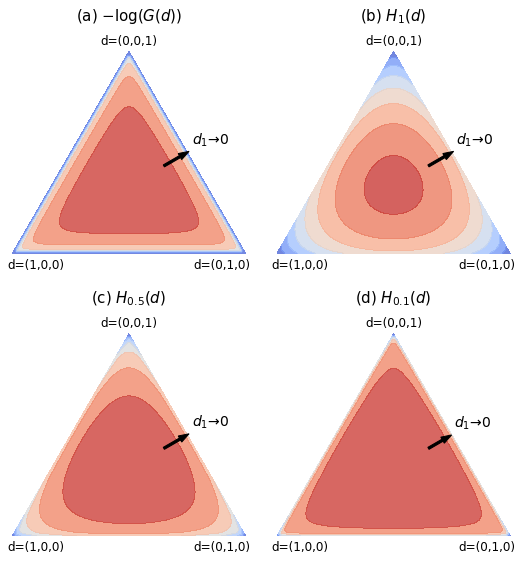}    
   \caption{
   The contours of different objective functions on $\Delta^n$ with $n=3$.
   $G(d)$ is the average size of the transition dataset sampled from the distribution $d$ to ensure a successful planning.
   Compared with Shannon entropy ($H_1$), the contour of R\'enyi entropy ($H_{0.5}$, $H_{0.1}$) aligns better with $G$ since R\'enyi entropy rapidly decreases when any component of $d$ approaches zero (indicated by the arrow).
   }
   \label{fig:contour}
\end{figure}

\textbf{Theoretical justification for the objective function.}
Next, we formally justify the use of R\'enyi entropy over the state-action space with the following theorem. For ease of analysis, we consider a standard episodic setting: 
The MDP has a finite planning horizon $H$ and stochastic dynamics $\mathbb{P}$ with the objective to maximize the cumulative reward $J(\pi; r) := \mathbb{E}_{s_1\sim \mu}[V^\pi(s_1;r)]$, where $V^\pi(s;r) := \mathbb{E}_\mathbb{P} \left[\sum_{h=1}^H r_h(s_h, a_h)|s_1=s, \pi \right]$.
We assume the reward function $r_h:\mathcal{S}\times\mathcal{A}\to[0,1], \forall h\in[H]$ is deterministic.
A (non-stationary, stochastic) policy $\pi: \mathcal{S}\times [H] \to \Delta^\mathcal{A}$ specifies the probability of choosing the actions on each state and on each step. 
The state-action distribution induced by $\pi$ on the $h$-th step is $d_{h}^\pi (s,a) := \text{Pr}(s_h=s, a_h=a|s_1\sim \mu; \pi)$.


\begin{theorem}
\label{theorem:main}
Denote $\varpi$ as a set of policies $\{ \pi^{(h)} \}_{h=1}^H$, where $\pi^{(h)}: S\times[H]\to \Delta^\mathcal{A}$ and $\pi^{(h)} \in arg \max_\pi H_\alpha(d_h^\pi)$.
Construct a dataset $\mathcal M$ with $M$ trajectories, each of which is collected by first uniformly randomly choosing a policy $\pi$ from $\varpi$ and then executing the policy $\pi$. Assume
\[
M \ge c \left( \dfrac{H^2SA}{\epsilon} \right)^{2(\beta + 1)} \frac{H}{A} \log \left( \dfrac{SAH}{p\epsilon} \right),
\]
where $\beta = \frac{\alpha}{2(1-\alpha)}$ and $c>0$ is an absolute constant.
Then, there exists a planning algorithm such that, for any reward function $r$, with probability at least $1-p$, the output policy $\hat \pi$ of the planning algorithm based on $\mathcal M$ is $3\epsilon$-optimal,  i.e., $J(\pi^*;r) - J(\hat{\pi};r) \le 3\epsilon$, where $J(\pi^*;r) = \max_\pi J(\pi;r)$.
\end{theorem}


We provide the proof in Appendix \ref{app:proof}.
The theorem justifies that R\'enyi entropy with small $\alpha$ is a proper objective function for the exploration phase since the number of samples needed to ensure a successful planning is bounded when $\alpha$ is small.
When $\alpha\to 1$, the bound becomes infinity. 
The algorithm proposed by \citet{jin2020reward} requires to sample $\tilde{O}(H^5S^2A/\epsilon^2)$ trajectories where $\tilde{O}$ hides a logarithmic factor, which matches our bound when $\alpha \to 0$.
However, they construct a policy for each state on each step, whereas we only need $H$ policies which easily adapts for the non-tabular case.

\section{Method}
\label{sec:method}

In this section, we develop an algorithm for the non-tabular case.
In the exploration phase, we update the policy $\pi$ to maximize $H_\alpha(d_\mu^\pi)$.
We first deduce a gradient ascent update rule which is similar to vanilla policy gradient with the reward replaced by a function of the state-action distribution of the current policy.
Afterwards, we estimate the state-action distribution using VAE.
We also estimate a value function and update the policy using PPO, which is more sample efficient and robust than vanilla policy gradient.
Then, we obtain a dataset by collecting samples from the learned policy.
In the planning phase, we use a popular batch RL algorithm, BCQ,
to plan for a good policy when the reward function is specified.
One may also use other batch RL algorithms.
We show the pseudocode of the process in Algorithm \ref{algorithm}, the details of which are described in the following paragraphs.

\begin{figure*}[htbp]
\begin{subfigure}{.26\textwidth}
  \centering
  \includegraphics[width=\linewidth]{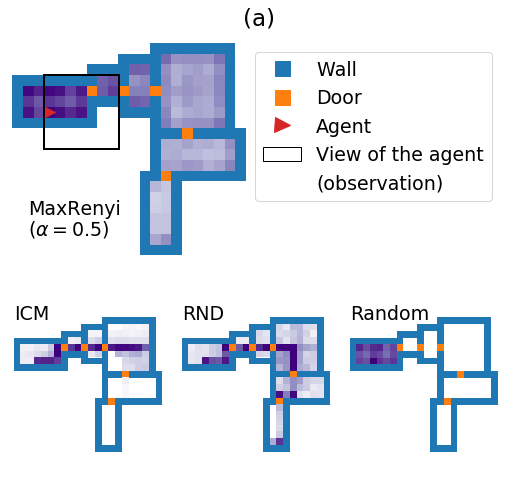}
\end{subfigure}
\begin{subfigure}{.34\textwidth}
  \centering
  \includegraphics[width=\linewidth]{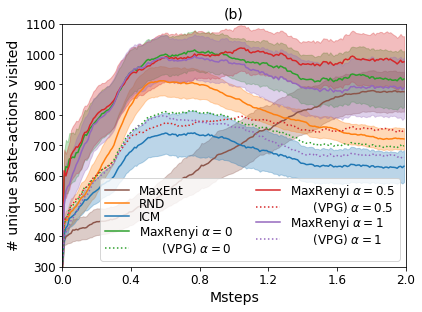}  
\end{subfigure}
\begin{subfigure}{.375\textwidth}
  \centering
  \includegraphics[width=\linewidth]{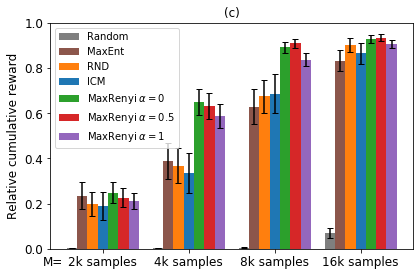}
\end{subfigure}
\caption{Experiments on \texttt{MultiRooms}. 
a) Illustration of the \texttt{MultiRooms} environment and the state visitation distribution of the policies learned by different algorithms. 
b) The number of unique state-action pairs visited by running the policy for 2k steps in each iteration in the exploration phase. 
c) The normalized cumulative reward of the policy computed in the planning phase (i.e., normalized by the corresponding optimal cumulative reward), averaged over different reward functions.
The lines indicate the average across five independent runs and the shaded areas indicate the standard deviation.
}
\label{fig:multiroom}
\end{figure*}



\textbf{Policy gradient formulation.} 
Let us first consider the gradient of the objective function $H_\alpha(d_\mu^{\pi})$, where the policy $\pi$ is approximated by a policy network with the parameters denoted as $\theta$.
We omit the dependency of $\pi$ on $\theta$.
The gradient of the objective function is
\begin{equation}
\label{eq:PG_alpha}
\begin{aligned}
    \nabla_\theta H_\alpha(d_\mu^\pi) & \propto \dfrac{\alpha}{1-\alpha} \mathbb{E}_{(s,a)\sim d_\mu^\pi} \Big[ \nabla_\theta \log \pi(a|s) \\
    &
    \left( \tfrac{1}{1-\gamma} \langle d_{s,a}^\pi, (d_\mu^\pi)^{\alpha-1} \rangle + (d_\mu^\pi (s,a))^{\alpha-1} \right) \Big].
\end{aligned}
\end{equation}
As a special case, when $\alpha\to 1$, the R\'enyi entropy becomes the Shannon entropy and the gradient turns into
\begin{equation}
\label{eq:PG_Shannon}
\begin{aligned}
    \nabla_\theta H_1(d_\mu^\pi) & = \mathbb{E}_{(s,a)\sim d_\mu^\pi} \Big[ \nabla_\theta \log \pi(a|s) \\
    & \left( \tfrac{1}{1-\gamma} \langle d_{s,a}^\pi, -\log d_\mu^\pi \rangle - \log d_\mu^\pi(s,a)
    \right) \Big].
\end{aligned}
\end{equation}
Due to space limit, the derivation can be found in Appendix \ref{app:PG}.
There are two terms in the gradients.
The first term equals to $\mathbb{E}_{(s,a)\sim d_\mu^\pi} \left[ \nabla_\theta \log \pi(a|s) Q^\pi(s,a;r) \right]$ with the reward $r(s,a) = (d_\mu^\pi(s,a))^{\alpha-1}$ or $r(s,a) = -\log d_\mu^\pi(s,a)$, which resembles the policy gradient (of cumulative reward) for a standard MDP.
This term encourages the policy to choose the actions that lead to the rarely visited state-action pairs.
In a similar way, the second term resembles the policy gradient of instant reward $\mathbb{E}_{(s,a)\sim d_\mu^\pi} \left[ \nabla_\theta \log \pi(a|s) r(s,a) \right]$ and encourages the agent to choose the actions that are rarely selected on the current step.
The second term in \eqref{eq:PG_Shannon} equals to 
$ \nabla_\theta \mathbb{E}_{s\sim d_\mu^\pi}[H_1(\pi(\cdot|s))]$.
For stability, we also replace the second term in \eqref{eq:PG_alpha} with $ \nabla_\theta \mathbb{E}_{s\sim d_\mu^\pi}[H_1(\pi(\cdot|s))]$
\footnote{
We found that using this term leads to more stable performance empirically since this term does not suffer from the high variance induced by the estimation of $d_\mu^\pi$ (see also Appendix \ref{app:explain}).
}.
Accordingly, we update the policy based on the following formula 
where $\eta>0$ is a hyperparameter:
\begin{equation}
\begin{aligned}
     \nabla_\theta H_\alpha(d_\mu^\pi) &
     \approx 
     \begin{cases}
     \nabla_\theta J(\theta;r=(d_\mu^\pi)^{\alpha-1})
     & 0<\alpha<1 \\
     \nabla_\theta J(\theta;r=-\log d_\mu^\pi) 
     & \alpha=1 \\
    \end{cases} \\
    & \qquad \qquad \quad +
     \eta 
     \nabla_\theta
     \mathbb{E}_{s\sim d_\mu^\pi}[H_1(\pi(\cdot|s))].
\end{aligned}
\end{equation}

\textbf{Discussion.}
\citet{islam2019entropy} motivate the agent to explore by maximizing the Shannon entropy over the state space resulting in an intrinsic reward $r(s)=-\log d_\mu^\pi(s)$ which is similar to ours when $\alpha\to 1$.
\citet{bellemare2016unifying} use an intrinsic reward $r(s,a)=\hat{N}(s,a)^{-1/2}$ where $\hat{N}(s,a)$ is an estimation of the visit count of $(s,a)$. Our algorithm with $\alpha = 1/2$ induces a similar reward.


\begin{algorithm}[t]
\caption{Maximizing the state-action space R\'enyi entropy for the reward-free RL framework}
\label{algorithm}
\begin{algorithmic}[1]
\State {\bf Input:} The number of iterations in the exploration phase $T$;
The size of the dataset $M$;
The parameter of R\'enyi entropy $\alpha$.
\State {\bf Initialize:}
Replay buffer that stores the samples form the last $n$ iterations $\mathcal{D}=\emptyset$; 
Density model $P_\phi$ (VAE); 
Value function $V_\psi$;
Policy network $\pi_\theta$.
\State{$\triangleright$ \emph{Exploration phase} (MaxRenyi)}
\For {$i=1$ to $T$}
    \State Sample $\mathcal{D}_i$ using $\pi_\theta$ and add it to $\mathcal{D}$
    \State Update $P_\phi$ to 
    estimate the state-action distribution based on $\mathcal{D}_i$
    \State Update $V_\psi$ to minimize $\mathcal{L}_\psi(\mathcal{D})$ defined in \eqref{eq:value_function_loss} 
    \State Update $\pi_\theta$ to minimize $\mathcal{L}_\theta(\mathcal{D}_i)$ defined in \eqref{eq:policy_network_loss} 
\EndFor
\State{$\triangleright$ \emph{Collect the dataset}}
\State Rollout the exploratory policy $\pi_\theta$, collect $M$ trajectories and store them in $\mathcal{M}$
\State{$\triangleright$ \emph{Planning phase}}
\State Reveal a reward function $r: \mathcal{S}\times \mathcal{A} \to \mathbb{R}$
\State Construct a labeled dataset $\overline{\mathcal{M}}:=\{(s,a,r(s,a))\}$
\State Plan for a policy $\pi_r = \text{BCQ}(\overline{\mathcal{M}})$
\end{algorithmic}
\end{algorithm}

\textbf{Sample collection.}
To estimate $d_\mu^\pi$ for calculating the underlying reward, we collect samples in the following way (cf. Line 5 in Algorithm \ref{algorithm}):
In the $i$-th iteration, we sample $m$ trajectories. 
In each trajectory, we terminate the rollout on each step with probability $1-\gamma$.
In this way, we obtain a set of trajectories $\mathcal{D}_i=\{(s_{j1},a_{j1}), \cdots, (s_{jN_j}, a_{jN_j})\}_{j=1}^m$ where $N_j$ is the length of the $j$-th trajectory.
Then, we can use VAE to estimate $d_\mu^\pi$ based on $\mathcal{D}_i$, i.e., using ELBO as the density estimation
(cf. Line 6 in Algorithm \ref{algorithm}).


\textbf{Value function.} Instead of performing vanilla policy gradient, we update the policy using PPO which is more robust and sample efficient. 
However, the underlying reward function changes across iterations. 
This makes it hard to learn a value function incrementally that is used to reduce variance in PPO.
We propose to train a value function network using relabeled off-policy data. 
In the $i$-th iteration, we maintain a replay buffer $\mathcal{D} = \mathcal{D}_i \cup \mathcal{D}_{i-1} \cup \cdots \cup \mathcal{D}_{i-n+1}$ that stores the trajectories of the last $n$ iterations (cf. Line 5 in Algorithm \ref{algorithm}).
Next, we calculate the reward for each state-action pair in $\mathcal{D}$ with the latest density estimator $P_\phi$, 
i.e.,
we assign $r = (P_\phi(s,a))^{\alpha-1}$ or $r = -\log P_\phi(s,a)$ for each $(s,a) \in \mathcal{D}$.
Based on these rewards, we can estimate the target value $V^\text{targ}(s)$ for each state $s\in\mathcal{D}$ using the truncated TD($\lambda$) estimator \cite{sutton2018introduction} which balances bias and variance (see the detail in Appendix \ref{app:value}).
Then, we train the value function network $V_\psi$ (where $\psi$ is the parameter) to minimize the mean squared error w.r.t. the target values:
\begin{equation}
\label{eq:value_function_loss}
    \mathcal{L}_\psi(\mathcal{D}) = \mathbb{E}_{s\sim\mathcal{D}}\left[(V_\psi(s) - V^\text{targ}(s))^2 \right].
\end{equation}

\textbf{Policy network.}
In each iteration, we update the policy network to maximize the following objective function that is used in PPO:
\begin{equation}
\label{eq:policy_network_loss}
\begin{aligned}
    -\mathcal{L}_\theta(\mathcal{D}_i) = & \mathbb{E}_{s,a\sim \mathcal{D}_i} \left[ 
    \min \left(
    \tfrac{\pi_\theta(a|s)}{\pi_\text{old}(a|s)} \hat{A}(s,a), 
    g(\varepsilon, \hat{A}(s,a))
    \right)
    \right] \\
    & + \eta \mathbb{E}_{s\sim \mathcal{D}_i}[H_1(\pi_\theta(\cdot|s))],
\end{aligned}
\end{equation}
where $
g(\varepsilon, A) = 
    \begin{cases}
    (1+\varepsilon) A & A \ge 0 \\
    (1-\varepsilon) A & A < 0
    \end{cases}
$ and $\hat{A}(s,a)$ is the advantage estimated using GAE \cite{schulman2015high} and the learned value function $V_\phi$.

\begin{figure*}[hbtp]
\begin{subfigure}{.3\textwidth}
  \centering
  \includegraphics[width=\linewidth]{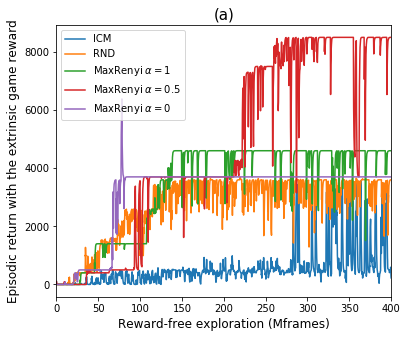}
\end{subfigure}
\begin{subfigure}{.68\textwidth}
  \centering
  \includegraphics[width=\linewidth]{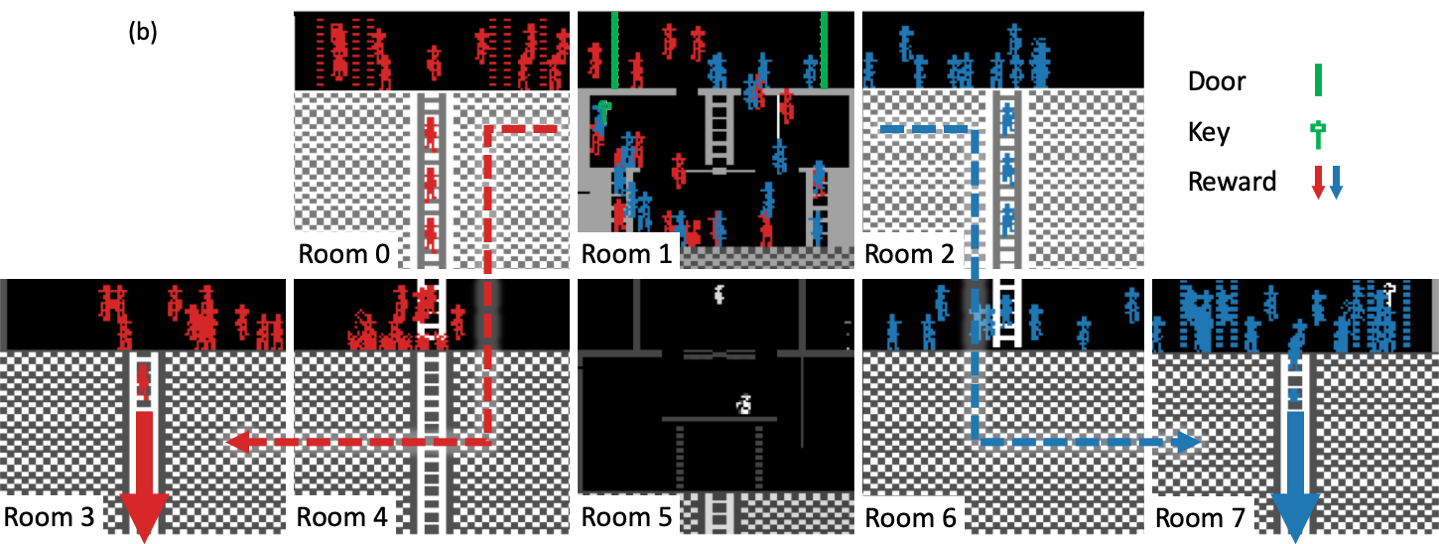}
\end{subfigure}
\caption{Experiments on Montezuma's Revenge. a) The episodic return with the extrinsic game reward along with the training in the exploration phase. b) The trajectory of the agent by executing the policy trained with the corresponding reward function in the planning phase. Best viewed in color.}
\label{fig:montezuma}
\end{figure*}

\section{Experiments}
\label{sec:experiment}

We first conduct experiments on the MultiRooms environment from minigrid \cite{gym_minigrid} to compare the performance of MaxRenyi with several baseline exploration algorithms in the exploration phase and the downstream planning phase.
We also compare the performance of MaxRenyi with different $\alpha$s, and show the advantage of using a PPO based algorithm and maximizing the entropy over the state-action space by comparing with ablated versions.
Then, we conduct experiments on a set of Atari (with image-based observations) \cite{machado18arcade} and Mujoco \cite{todorov2012mujoco} tasks to show that our algorithm outperforms the baselines
in complex environments for the reward-free RL framework.
We also provide detailed results to investigate why our algorithm succeeds.
More experiments and the detailed experiment settings/hyperparameters can be found in Appendix \ref{app:exp}.

\textbf{Experiments on MultiRooms.}
The observation in MultiRooms is the first person perspective from the agent which is high-dimensional and partially observable, and the actions are turning left/right, moving forward and opening the door (cf. Figure \ref{fig:multiroom}a above).
This environment is hard to explore for standard RL algorithms due to sparse reward.
In the exploration phase, the agent has to navigate and open the doors to explore through a series of rooms. 
In the planning phase, we randomly assign a goal state-action pair, reward the agent upon this state-action, and then train a policy with this reward function.
We compare our exploration algorithm MaxRenyi with 
ICM \cite{pathak2017curiosity},
RND \cite{burda2018exploration} (which use different prediction errors as the intrinsic reward), 
MaxEnt \cite{hazan2018provably} (which maximizes the state space entropy) and an ablated version of our algorithm MaxRenyi(VPG) (that updates the policy directly by vanilla policy gradient using \eqref{eq:PG_alpha} and \eqref{eq:PG_Shannon}). 

For the exploration phase, we show the performance of different algorithms in Figure \ref{fig:multiroom}b.
First, we see that variants of MaxRenyi performs better than the baseline algorithms.
Specifically, MaxRenyi is more stable than ICM and RND that
explores well at the start but later degenerates when the agent becomes familiar with all the states.
Second, we observe that 
MaxRenyi 
performs better than 
MaxRenyi(VPG) with different $\alpha$s,
indicating that MaxRenyi benefits from
adopting PPO which reduces variance with a value function and update the policy conservatively.
Third, MaxRenyi with $\alpha=0.5$ performs better than $\alpha=1$ which is consistent with our theory.
However, we observe that MaxRenyi with $\alpha=0$ is more likely to be unstable empirically and results in slightly worse performance.
At last, we show the visitation frequency of the exploratory policies resulted from different algorithms in Figure \ref{fig:multiroom}a.
We see that MaxRenyi visits the states more uniformly compared with the baseline algorithms, especially the hard-to-reach states such as the corners of the rooms.

For the planning phase, we collect datasets of different sizes by executing different exploratory policies, and use the datasets to compute policies with different reward functions using BCQ.
We show the performance of the resultant policies in Figure \ref{fig:multiroom}c.
First, we observe that the datasets generated by running random policies do not lead to a successful planning, indicating the importance of learning a good exploratory policy in this framework. 
Second, 
the dataset with only 8k samples leads to a successful planning (with a normalized cumulative reward larger than 0.8) using MaxRenyi,
whereas a dataset with 16k samples is needed to succeed in the planning phase when using ICM, RND or MaxEnt.
This illustrates that MaxRenyi leads to a better performance in the planning phase (i.e., attains good policies with fewer samples) than the previous exploration algorithms.

Further, we compare MaxRenyi that maximizes the entropy over the state-action space with two ablated versions: ``State + bonus''
that maximizes the entropy over the state space with an entropy bonus (i.e., using $P_\phi$ to estimate the state distribution in Algorithm \ref{algorithm}), and
``State'' that maximizes the entropy over the state space (i.e., additionally setting $\eta=0$ in \eqref{eq:policy_network_loss}).
Notice that the reward-free RL framework requires that a good policy is obtained for arbitrary reward functions in the planning phase.
Therefore, we show the lowest normalized cumulative reward of the planned policies across different reward functions for the algorithms in Table \ref{table:2}.
We see that maximizing the entropy over the state-action space results in better performance for this framework. 

\begin{table}[h]
\centering
\begin{tabular}{lcccc} 
\multicolumn{1}{r}{$M=$} & 2k & 4k & 8k & 16k \\
\hline
MaxRenyi & \textbf{0.14} & \textbf{0.52} & \textbf{0.87} & \textbf{0.91} \\
State + bonus & \textbf{0.14} & 0.48 & 0.78 & 0.87 \\
State & 0.07 & 0.15 & 0.41 & 0.57 \\
\end{tabular}
\caption{The lowest normalized cumulative reward of the planned policies across different reward functions, averaged over five runs. We use $\alpha=0.5$ in the algorithms.}
\label{table:2}
\end{table}

\textbf{Experiments on Atari and Mujoco.}
In the exploration phase, we run different exploration algorithms in the reward-free environment of Atari (Mujoco) for 200M (10M) steps and collect a dataset with 100M (5M) samples by executing the learned policy.
In the planning phase, a policy is computed offline to maximize the performance under the extrinsic game reward using the batch RL algorithm based on the dataset.
We show the performance of the planned policies in Table \ref{table:1}.
We can see that MaxRenyi performs well on a range of tasks with high-dimensional observations or continuous state-actions for the reward-free RL framework.

\begin{table}[h]
\centering
\begin{tabular}{rrrrr} 
& MaxRenyi & RND & ICM & Random \\
\hline
Montezuma & \textbf{8,100} & 4,740 & 380 & 0 \\
Venture & \textbf{900} & 840 & 240 & 0 \\
Gravitar & 1,850 & \textbf{1,890} & 940 & 340 \\
PrivateEye & \textbf{6,820} & 3,040 & 100 & 100 \\
Solaris & \textbf{2,192} & 1,416 & 1,844 & 1,508 \\
Hopper & \textbf{973} & 819 & 704 & 78 \\
HalfCheetah & \textbf{1,466} & 1,083 & 700 & 54 \\
\end{tabular}
\caption{The performance of the planned policies under the reward-free RL framework corresponding to different exploratory algorithms, averaged over five runs. Random indicates using a random policy as the exploratory policy. We use $\alpha=0.5$ in MaxRenyi.}
\label{table:1}
\end{table}

We also provide detailed results for Montezuma's Revenge and Hopper.
For Montezuma's Revenge, we show the result for the exploration phase in Figure \ref{fig:montezuma}a.
We observe that although MaxRenyi 
learns an exploratory policy in a reward-free environment,
it achieves reasonable performance under the extrinsic game reward and performs better than RND and ICM.
We also provide the number of visited rooms along the training in Appendix \ref{app:exp}, which demonstrates that our algorithm performs better in terms of the state space coverage as well.
In the planning phase, we design two different sparse reward functions that only reward the agent if it goes through Room 3 or Room 7 (to the next room).
We show the trajectories of the policies planned with the two reward functions in red and blue respectively in Figure \ref{fig:montezuma}b.
We see that although the reward functions are sparse, the agent chooses the correct path (e.g., opening the correct door in Room 1 with the only key) and successfully reaches the specified room.
This indicates that our algorithm generates good policies for different reward functions based on an offline planning in complex environments.
For Hopper, we show the t-SNE \cite{maaten2008visualizing} plot of the state-actions randomly sampled from the trajectories generated by different policies in Figure \ref{fig:hopper}.
We can see that MaxRenyi generates state-actions that are distributed more uniformly than RND and overlap those from the expert policy.
This indicates that
the good performance of our algorithm in the planning phase is resulted from the better coverage of our exploratory policy.

\begin{figure}[thbp]
  \centering
  \includegraphics[trim={0 0 0 30},clip,width=0.7\columnwidth]{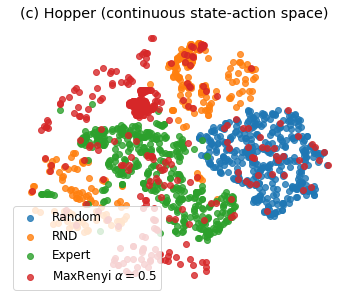}
\caption{t-SNE plot of the state-actions randomly sampled from the trajectories generated by different policies in Hopper. 
Random policy chooses actions randomly. 
Expert policy is trained using PPO for 2M steps with extrinsic reward. 
RND and MaxRenyi policy are trained using corresponding algorithms without extrinsic reward. }
\label{fig:hopper}
\end{figure}

\vspace{-10pt}
\section{Conclusion}

In this paper, we consider a reward-free RL framework, 
which is useful when there are multiple reward functions of interest or when we design reward functions offline.
In this framework, an exploratory policy is learned by interacting with a reward-free environment in the exploration phase and generates a dataset.
In the planning phase, when the reward function is specified, a policy is computed offline to maximize the corresponding cumulative reward using the batch RL algorithm based on the dataset.
We propose a novel objective function, the R\'enyi entropy over the state-action space, for the exploration phase.
We theoretically justify this objective and design a practical algorithm to optimize for this objective.
In the experiments, we show that our exploration algorithm is effective under this framework, while being more sample efficient and more robust than the previous exploration algorithms.









\section*{Acknowledgement}
Chuheng Zhang and Jian Li are supported in part by the National Natural Science Foundation of China Grant 61822203, 61772297, 61632016, 61761146003,
and the Zhongguancun Haihua Institute for Frontier Information Technology, Turing AI Institute of Nanjing and Xi'an Institute for Interdisciplinary Information Core Technology.
Yuanying Cai and Longbo Huang are supported in part by the National Natural Science Foundation of China Grant 61672316, and the Zhongguancun Haihua Institute for Frontier Information Technology, China.

\bibliography{reference}

\begin{thebibliography}{51}
\providecommand{\natexlab}[1]{#1}
\providecommand{\url}[1]{\texttt{#1}}
\providecommand{\urlprefix}{URL }
\expandafter\ifx\csname urlstyle\endcsname\relax
  \providecommand{\doi}[1]{doi:\discretionary{}{}{}#1}\else
  \providecommand{\doi}{doi:\discretionary{}{}{}\begingroup
  \urlstyle{rm}\Url}\fi

\bibitem[{Abdolmaleki et~al.(2018)Abdolmaleki, Springenberg, Tassa, Munos,
  Heess, and Riedmiller}]{abdolmaleki2018maximum}
Abdolmaleki, A.; Springenberg, J.~T.; Tassa, Y.; Munos, R.; Heess, N.; and
  Riedmiller, M. 2018.
\newblock Maximum a posteriori policy optimisation.
\newblock In \emph{Proceedings of the 6th International Conference on Learning
  Representations}.

\bibitem[{Achiam and Sastry(2016)}]{achiam2017surprise}
Achiam, J.; and Sastry, S. 2016.
\newblock Surprise-based intrinsic motivation for deep reinforcement learning.
\newblock In \emph{Deep RL Workshop, Advances in Neural Information Processing
  Systems}.

\bibitem[{Agarwal et~al.(2019)Agarwal, Kakade, Lee, and
  Mahajan}]{agarwal2019optimality}
Agarwal, A.; Kakade, S.~M.; Lee, J.~D.; and Mahajan, G. 2019.
\newblock Optimality and approximation with policy gradient methods in markov
  decision processes.
\newblock \emph{arXiv preprint arXiv:1908.00261} .

\bibitem[{Auer, Cesa-Bianchi, and Fischer(2002)}]{auer2002finite}
Auer, P.; Cesa-Bianchi, N.; and Fischer, P. 2002.
\newblock Finite-time analysis of the multiarmed bandit problem.
\newblock \emph{Machine Learning} 47(2-3): 235--256.

\bibitem[{Bellemare et~al.(2016)Bellemare, Srinivasan, Ostrovski, Schaul,
  Saxton, and Munos}]{bellemare2016unifying}
Bellemare, M.; Srinivasan, S.; Ostrovski, G.; Schaul, T.; Saxton, D.; and
  Munos, R. 2016.
\newblock Unifying count-based exploration and intrinsic motivation.
\newblock In \emph{Advances in Neural Information Processing Systems},
  1471--1479.

\bibitem[{Brockman et~al.(2016)Brockman, Cheung, Pettersson, Schneider,
  Schulman, Tang, and Zaremba}]{1606.01540}
Brockman, G.; Cheung, V.; Pettersson, L.; Schneider, J.; Schulman, J.; Tang,
  J.; and Zaremba, W. 2016.
\newblock OpenAI Gym.

\bibitem[{Bromiley, Thacker, and Bouhova-Thacker(2004)}]{bromiley2004shannon}
Bromiley, P.; Thacker, N.; and Bouhova-Thacker, E. 2004.
\newblock Shannon entropy, Renyi entropy, and information.
\newblock \emph{Statistics and Inf. Series (2004-004)} .

\bibitem[{Burda et~al.(2019{\natexlab{a}})Burda, Edwards, Pathak, Storkey,
  Darrell, and Efros}]{burda2018large}
Burda, Y.; Edwards, H.; Pathak, D.; Storkey, A.; Darrell, T.; and Efros, A.~A.
  2019{\natexlab{a}}.
\newblock Large-scale study of curiosity-driven learning.
\newblock In \emph{Proceedings of the 7th International Conference on Learning
  Representations}.

\bibitem[{Burda et~al.(2019{\natexlab{b}})Burda, Edwards, Storkey, and
  Klimov}]{burda2018exploration}
Burda, Y.; Edwards, H.; Storkey, A.; and Klimov, O. 2019{\natexlab{b}}.
\newblock Exploration by random network distillation.
\newblock In \emph{Proceedings of the 7th International Conference on Learning
  Representations}.

\bibitem[{Chevalier-Boisvert, Willems, and Pal(2018)}]{gym_minigrid}
Chevalier-Boisvert, M.; Willems, L.; and Pal, S. 2018.
\newblock Minimalistic Gridworld Environment for OpenAI Gym.
\newblock \url{https://github.com/maximecb/gym-minigrid}.

\bibitem[{Conti et~al.(2018)Conti, Madhavan, Such, Lehman, Stanley, and
  Clune}]{conti2018improving}
Conti, E.; Madhavan, V.; Such, F.~P.; Lehman, J.; Stanley, K.; and Clune, J.
  2018.
\newblock Improving exploration in evolution strategies for deep reinforcement
  learning via a population of novelty-seeking agents.
\newblock In \emph{Advances in Neural Information Processing Systems},
  5027--5038.

\bibitem[{Dann, Lattimore, and Brunskill(2017)}]{dann2017unifying}
Dann, C.; Lattimore, T.; and Brunskill, E. 2017.
\newblock Unifying PAC and regret: Uniform PAC bounds for episodic
  reinforcement learning.
\newblock In \emph{Advances in Neural Information Processing Systems},
  5713--5723.

\bibitem[{de~Abril and Kanai(2018)}]{de2018curiosity}
de~Abril, I.~M.; and Kanai, R. 2018.
\newblock Curiosity-driven reinforcement learning with homeostatic regulation.
\newblock In \emph{Proceedings of the International Joint Conference on Neural
  Networks (IJCNN)}, 1--6. IEEE.

\bibitem[{Duan et~al.(2016)Duan, Schulman, Chen, Bartlett, Sutskever, and
  Abbeel}]{duan2016rl}
Duan, Y.; Schulman, J.; Chen, X.; Bartlett, P.~L.; Sutskever, I.; and Abbeel,
  P. 2016.
\newblock RL2: Fast reinforcement learning via slow reinforcement learning.
\newblock \emph{arXiv preprint arXiv:1611.02779} .

\bibitem[{Ecoffet et~al.(2019)Ecoffet, Huizinga, Lehman, Stanley, and
  Clune}]{ecoffet2019go}
Ecoffet, A.; Huizinga, J.; Lehman, J.; Stanley, K.~O.; and Clune, J. 2019.
\newblock Go-explore: a new approach for hard-exploration problems.
\newblock \emph{arXiv preprint arXiv:1901.10995} .

\bibitem[{Eysenbach et~al.(2019)Eysenbach, Gupta, Ibarz, and
  Levine}]{eysenbach2018diversity}
Eysenbach, B.; Gupta, A.; Ibarz, J.; and Levine, S. 2019.
\newblock Diversity is all you need: Learning skills without a reward function.
\newblock In \emph{Proceedings of the 7th International Conference on Learning
  Represention}.

\bibitem[{Finn, Abbeel, and Levine(2017)}]{finn2017model}
Finn, C.; Abbeel, P.; and Levine, S. 2017.
\newblock Model-agnostic meta-learning for fast adaptation of deep networks.
\newblock In \emph{Proceedings of the 34th International Conference on Machine
  Learning}, volume~70, 1126--1135. JMLR.

\bibitem[{Flajolet, Gardy, and Thimonier(1992)}]{flajolet1992birthday}
Flajolet, P.; Gardy, D.; and Thimonier, L. 1992.
\newblock Birthday paradox, coupon collectors, caching algorithms and
  self-organizing search.
\newblock \emph{Discrete Applied Mathematics} 39(3): 207--229.

\bibitem[{Fujimoto et~al.(2019)Fujimoto, Conti, Ghavamzadeh, and
  Pineau}]{fujimoto2019benchmarking}
Fujimoto, S.; Conti, E.; Ghavamzadeh, M.; and Pineau, J. 2019.
\newblock Benchmarking Batch Deep Reinforcement Learning Algorithms.
\newblock In \emph{Deep RL Workshop, Advances in Neural Information Processing
  Systems}.

\bibitem[{Fujimoto, Meger, and Precup(2019)}]{fujimoto2018off}
Fujimoto, S.; Meger, D.; and Precup, D. 2019.
\newblock Off-Policy Deep Reinforcement Learning without Exploration.
\newblock In \emph{Proceedings of the 36th International Conference on Machine
  Learning}, volume~97, 2052--2062. JMLR.

\bibitem[{Gupta et~al.(2018)Gupta, Eysenbach, Finn, and
  Levine}]{gupta2018unsupervised}
Gupta, A.; Eysenbach, B.; Finn, C.; and Levine, S. 2018.
\newblock Unsupervised meta-learning for reinforcement learning.
\newblock \emph{arXiv preprint arXiv:1806.04640} .

\bibitem[{Haarnoja et~al.(2017)Haarnoja, Tang, Abbeel, and
  Levine}]{haarnoja2017reinforcement}
Haarnoja, T.; Tang, H.; Abbeel, P.; and Levine, S. 2017.
\newblock Reinforcement learning with deep energy-based policies.
\newblock In \emph{Proceedings of the 34th International Conference on Machine
  Learning}, volume~70, 1352--1361. JMLR.

\bibitem[{Haarnoja et~al.(2018)Haarnoja, Zhou, Abbeel, and
  Levine}]{haarnoja2018soft}
Haarnoja, T.; Zhou, A.; Abbeel, P.; and Levine, S. 2018.
\newblock Soft actor-critic: Off-policy maximum entropy deep reinforcement
  learning with a stochastic actor.
\newblock In \emph{Proceedings of the 35th International Conference on Machine
  Learning}, volume~80, 1861--1870. JMLR.

\bibitem[{Hazan et~al.(2019)Hazan, Kakade, Singh, and
  Van~Soest}]{hazan2018provably}
Hazan, E.; Kakade, S.; Singh, K.; and Van~Soest, A. 2019.
\newblock Provably Efficient Maximum Entropy Exploration.
\newblock In \emph{Proceedings of the 36th International Conference on Machine
  Learning}, volume~97, 2681--2691. JMLR.

\bibitem[{Houthooft et~al.(2016)Houthooft, Chen, Duan, Schulman, De~Turck, and
  Abbeel}]{houthooft2016vime}
Houthooft, R.; Chen, X.; Duan, Y.; Schulman, J.; De~Turck, F.; and Abbeel, P.
  2016.
\newblock Vime: Variational information maximizing exploration.
\newblock In \emph{Advances in Neural Information Processing Systems},
  1109--1117.

\bibitem[{Islam et~al.(2019)Islam, Seraj, Bacon, and Precup}]{islam2019entropy}
Islam, R.; Seraj, R.; Bacon, P.-L.; and Precup, D. 2019.
\newblock Entropy Regularization with Discounted Future State Distribution in
  Policy Gradient Methods.
\newblock In \emph{Optimization Foundations of RL Workshop, Advances in Neural
  Information Processing Systems}.

\bibitem[{Jin et~al.(2020)Jin, Krishnamurthy, Simchowitz, and
  Yu}]{jin2020reward}
Jin, C.; Krishnamurthy, A.; Simchowitz, M.; and Yu, T. 2020.
\newblock Reward-Free Exploration for Reinforcement Learning.
\newblock In \emph{Proceedings of the 37th International Conference on Machine
  Learning}. JMLR.

\bibitem[{Kingma and Welling(2014)}]{kingma2013auto}
Kingma, D.~P.; and Welling, M. 2014.
\newblock Auto-encoding variational bayes.
\newblock In \emph{Proceedings of the 2nd International Conference on Learning
  Representations}.

\bibitem[{Lai and Robbins(1985)}]{lai1985asymptotically}
Lai, T.~L.; and Robbins, H. 1985.
\newblock Asymptotically efficient adaptive allocation rules.
\newblock \emph{Advances in Applied Mathematics} 6(1): 4--22.

\bibitem[{Lange, Gabel, and Riedmiller(2012)}]{lange2012batch}
Lange, S.; Gabel, T.; and Riedmiller, M. 2012.
\newblock Batch reinforcement learning.
\newblock In \emph{Reinforcement learning}, 45--73. Springer.

\bibitem[{Lehman and Stanley(2011)}]{lehman2011novelty}
Lehman, J.; and Stanley, K.~O. 2011.
\newblock Novelty search and the problem with objectives.
\newblock In \emph{Genetic programming theory and practice IX}, 37--56.
  Springer.

\bibitem[{Maaten and Hinton(2008)}]{maaten2008visualizing}
Maaten, L. v.~d.; and Hinton, G. 2008.
\newblock Visualizing data using t-SNE.
\newblock \emph{Journal of machine learning research} 9(Nov): 2579--2605.

\bibitem[{Machado et~al.(2018)Machado, Bellemare, Talvitie, Veness, Hausknecht,
  and Bowling}]{machado18arcade}
Machado, M.~C.; Bellemare, M.~G.; Talvitie, E.; Veness, J.; Hausknecht, M.~J.;
  and Bowling, M. 2018.
\newblock Revisiting the Arcade Learning Environment: Evaluation Protocols and
  Open Problems for General Agents.
\newblock \emph{Journal of Artificial Intelligence Research} 61: 523--562.

\bibitem[{Misra et~al.(2019)Misra, Henaff, Krishnamurthy, and
  Langford}]{misra2019kinematic}
Misra, D.; Henaff, M.; Krishnamurthy, A.; and Langford, J. 2019.
\newblock Kinematic State Abstraction and Provably Efficient Rich-Observation
  Reinforcement Learning.
\newblock \emph{arXiv preprint arXiv:1911.05815} .

\bibitem[{Mohamed and Rezende(2015)}]{mohamed2015variational}
Mohamed, S.; and Rezende, D.~J. 2015.
\newblock Variational information maximisation for intrinsically motivated
  reinforcement learning.
\newblock In \emph{Advances in Neural Information Processing Systems},
  2125--2133.

\bibitem[{Nachum et~al.(2017)Nachum, Norouzi, Xu, and
  Schuurmans}]{nachum2017bridging}
Nachum, O.; Norouzi, M.; Xu, K.; and Schuurmans, D. 2017.
\newblock Bridging the gap between value and policy based reinforcement
  learning.
\newblock In \emph{Advances in Neural Information Processing Systems},
  2775--2785.

\bibitem[{Nagabandi et~al.(2019)Nagabandi, Clavera, Liu, Fearing, Abbeel,
  Levine, and Finn}]{nagabandi2018learning}
Nagabandi, A.; Clavera, I.; Liu, S.; Fearing, R.~S.; Abbeel, P.; Levine, S.;
  and Finn, C. 2019.
\newblock Learning to adapt in dynamic, real-world environments through
  meta-reinforcement learning.
\newblock In \emph{Proceedings of the 7th International Conference on Learning
  Representations}.

\bibitem[{Ostrovski et~al.(2017)Ostrovski, Bellemare, van~den Oord, and
  Munos}]{ostrovski2017count}
Ostrovski, G.; Bellemare, M.~G.; van~den Oord, A.; and Munos, R. 2017.
\newblock Count-based exploration with neural density models.
\newblock In \emph{Proceedings of the 34th International Conference on Machine
  Learning-Volume 70}, 2721--2730. JMLR. org.

\bibitem[{Pathak et~al.(2017)Pathak, Agrawal, Efros, and
  Darrell}]{pathak2017curiosity}
Pathak, D.; Agrawal, P.; Efros, A.~A.; and Darrell, T. 2017.
\newblock Curiosity-driven exploration by self-supervised prediction.
\newblock In \emph{Proceedings of the IEEE Conference on Computer Vision and
  Pattern Recognition Workshops}, 16--17.

\bibitem[{Puterman(2014)}]{puterman2014markov}
Puterman, M.~L. 2014.
\newblock \emph{Markov decision processes: discrete stochastic dynamic
  programming}.
\newblock John Wiley \& Sons.

\bibitem[{Sanei, Borzadaran, and Amini(2016)}]{sanei2016renyi}
Sanei, M.; Borzadaran, G. R.~M.; and Amini, M. 2016.
\newblock Renyi entropy in continuous case is not the limit of discrete case.
\newblock \emph{Mathematical Sciences and Applications E-Notes} 4.

\bibitem[{Schmidhuber(1991)}]{schmidhuber1991possibility}
Schmidhuber, J. 1991.
\newblock A possibility for implementing curiosity and boredom in
  model-building neural controllers.
\newblock In \emph{Proceedings of the international conference on simulation of
  adaptive behavior: From animals to animats}, 222--227.

\bibitem[{Schulman, Chen, and Abbeel(2017)}]{schulman2017equivalence}
Schulman, J.; Chen, X.; and Abbeel, P. 2017.
\newblock Equivalence between policy gradients and soft q-learning.
\newblock \emph{arXiv preprint arXiv:1704.06440} .

\bibitem[{Schulman et~al.(2016)Schulman, Moritz, Levine, Jordan, and
  Abbeel}]{schulman2015high}
Schulman, J.; Moritz, P.; Levine, S.; Jordan, M.; and Abbeel, P. 2016.
\newblock High-dimensional continuous control using generalized advantage
  estimation.
\newblock In \emph{Proceedings of the 4th International Conference of Learning
  Representations}.

\bibitem[{Schulman et~al.(2017)Schulman, Wolski, Dhariwal, Radford, and
  Klimov}]{schulman2017proximal}
Schulman, J.; Wolski, F.; Dhariwal, P.; Radford, A.; and Klimov, O. 2017.
\newblock Proximal policy optimization algorithms.
\newblock \emph{arXiv preprint arXiv:1707.06347} .

\bibitem[{Silver(2015)}]{silver2015lecture}
Silver, D. 2015.
\newblock Lecture 2: Markov Decision Processes.
\newblock \emph{UCL. Retrieved from www0. cs. ucl. ac. uk/staff/d.
  silver/web/Teaching\_files/MDP. pdf} .

\bibitem[{Strehl and Littman(2008)}]{strehl2008analysis}
Strehl, A.~L.; and Littman, M.~L. 2008.
\newblock An analysis of model-based interval estimation for Markov decision
  processes.
\newblock \emph{Journal of Computer and System Sciences} 74(8): 1309--1331.

\bibitem[{Sutton and Barto(2018)}]{sutton2018introduction}
Sutton, R.~S.; and Barto, A.~G. 2018.
\newblock \emph{Reinforcement Learning: An Introduction}.
\newblock MIT press Cambridge.

\bibitem[{Todorov, Erez, and Tassa(2012)}]{todorov2012mujoco}
Todorov, E.; Erez, T.; and Tassa, Y. 2012.
\newblock Mujoco: A physics engine for model-based control.
\newblock In \emph{2012 IEEE/RSJ International Conference on Intelligent Robots
  and Systems}, 5026--5033. IEEE.

\bibitem[{Wang et~al.(2016)Wang, Kurth-Nelson, Tirumala, Soyer, Leibo, Munos,
  Blundell, Kumaran, and Botvinick}]{wang2016learning}
Wang, J.~X.; Kurth-Nelson, Z.; Tirumala, D.; Soyer, H.; Leibo, J.~Z.; Munos,
  R.; Blundell, C.; Kumaran, D.; and Botvinick, M. 2016.
\newblock Learning to reinforcement learn.
\newblock \emph{arXiv preprint arXiv:1611.05763} .

\bibitem[{Williams(1992)}]{williams1992simple}
Williams, R.~J. 1992.
\newblock Simple statistical gradient-following algorithms for connectionist
  reinforcement learning.
\newblock \emph{Machine Learning} 8(3-4): 229--256.

\end{thebibliography}

\clearpage

\begin{alphasection}
\onecolumn

\section{The toy example in Figure \ref{fig:five-state}}
\label{app:toy}

In this example, the initial state is fixed to be $s_1$.
Due to the simplicity of this toy example, we can solve for the policy that maximizes the R\'enyi entropy of the state-action distribution with $\gamma \to 1$ and $\alpha = 0.5$. 
The optimal policy is
\[
\pi = 
\begin{bmatrix}
0.321 & 0.276 & 0.294 & 0.401 & 0.381 \\
0.679 & 0.724 & 0.706 & 0.599 & 0.619 
\end{bmatrix},
\]
where $\pi_{i,j}$ represents the probability of choosing $a_i$ on $s_j$.
The corresponding state-action distribution is 
\[
d_\mu^\pi = 
\begin{bmatrix}
0.107 & 0.062 & 0.047 & 0.045 & 0.065 \\ 0.226 & 0.162 & 0.113 & 0.067 & 0.106
\end{bmatrix},
\]
where the $(i,j)$-th element represents the probability density of the state-action distribution on $(s_j, a_i)$.
Using equation \eqref{eq:coupon}, we obtain $G(d_\mu^\pi) = 43.14$ which is the expected number of samples collected form this distribution that contains all the state-action pairs.

\section{An example to illustrate the advantage of using R\'enyi entropy}
\label{app:CMP}


Although we did not justify theoretically that R\'enyi entropy serves as a better surrogate function of $G$ than Shannon entropy in the main text, here we illustrate that this may hold in general.
Specifically, we try to show that the corresponding $G$ of the policy that maximizes R\'enyi entropy with small $\alpha$ is always smaller than that of the policy that maximizes Shannon entropy.
Notice that when the dynamics is known to be deterministic, $G$ indicates how many samples are needed in expectation to ensure a successful planning.
In the stochastic setting, $G$ also serves as a good indicator since this is the lower bound of the number of samples needed for a successful planning.

To show that R\'enyi entropy is a good surrogate function, we can use brute-force search to check if there exists a CMP on which the Shannon entropy is better than R\'enyi entropy in terms of $G$.
As an illustration, we only consider  CMPs with two states and two actions. Specifically, for this family of CMPs, there are 5 free parameters to be searched: 4 parameters for the transition dynamics and 1 parameter for the initial state distribution. 
We set step size for searching to 0.1 for each parameter and use $\alpha=0.1$, $\gamma=0.9$.
The result shows there does not exists any CMP on which the Shannon entropy is better than R\'enyi entropy in term of $G$ out of the searched $10^5$ CMPs.




To further investigate,  we consider a specific CMP. The transition matrix $P\in\mathbb{R}^{SA\times S}$ of this CMP is
\[
P = 
\begin{bmatrix}
1.0 & 0.0 \\
0.9 & 0.1 \\
1.0 & 0.0 \\
0.4 & 0.6 
\end{bmatrix},
\] 
where the rows corresponds to $(s_1, a_1), (s_1, a_2), (s_2, a_1), (s_2, a_2)$ respectively and the columns corresponds the next states ($s_1$ and $s_2$ respectively).
The initial state distribution of this CMP is 
\[
\mu = 
\begin{bmatrix}
1.0 & 0.0
\end{bmatrix}.
\]
For this CMP, the $G$ value of the policy that maximizes Shannon entropy is 88, while the values are only 39 and 47 for policies that maximize R\'enyi entropy with $\alpha = 0.1$, $\alpha=0.5$ respectively. 
The optimal $G$ value is 32. 
This indicates that we may need 2x more samples to accurately estimate the dynamics of the CMP.
We observe $s_2$ is much harder to reach than $s_1$ on this CMP: 
The initial state is always $s_1$ and taking action $a_1$ will always reach $s_1$. 
In order to visit $s_2$ more, policies should take $a_2$ with a high probability on $s_1$. 
In fact, the policy that maximizes Shannon entropy takes $a_2$ on $s_1$ with the probability of only 0.58, while the policies that maximize R\'enyi entropy with $\alpha = 0.1$ and $\alpha=0.5$ take $a_2$ on $s_1$ with probabilities 0.71 and 0.64 respectively.  

This example illustrates the following essence that widely exists in other CMPs: 
A good objective function should encourage the policy to visit the hard-to-reach state-actions as frequently as possible and R\'enyi entropy does a better job than Shannon entropy from this perspective. 

\section{Proof of Theorem \ref{theorem:main}}
\label{app:proof}

We provide Theorem \ref{theorem:main} to justify our objective that maximizes the entropy over the state-action space in the exploration phase for the reward-free RL framework.
In the following analysis, we consider the standard episodic and tabular case as follows:
The MDP is defined as 
$\langle \mathcal{S}, \mathcal{A}, H, \mathbb{P}, r, \mu\rangle$, 
where $\mathcal{S}$ is the state space with $S=|\mathcal{S}|$, 
$\mathcal{A}$ is the action space with $A=|\mathcal{A}|$, 
$H$ is the length of each episode, 
$\mathbb{P}=\{\mathbb{P}_h: h\in [H]\}$ is the set of unknown transition matrices, 
$r=\{r_h: h\in[H]\}$ is the set of the reward functions and 
$\mu$ is the unknown initial state distribution. 
We denote $\mathbb{P}_h(s'|s,a)$ as the transition probability from $(s,a)$ to $s'$ on the $h$-th step.
We assume that the instant reward for taking the action $a$ on the state $s$ on the $h$-th step $r_h(s,a)\in [0,1]$ is deterministic.
A policy $\pi:\mathcal{S}\times[H] \to \Delta^\mathcal{A}$ specifies the probability of choosing the actions on each state and on each step.
We denote the set of all such policies as $\Pi$.

The agent with policy $\pi$ interacts with the MDP as follows: 
At the start of each episode, an initial state $s_1$ is sampled from the distribution $\mu$. 
Next, at each step $h\in[H]$, the agent first observes the state $s_h\in\mathcal{S}$ and then selects an action $a_h$ from the distribution $\pi(s_h, h)$.
The environment returns the reward $r_h(s_h, a_h) \in [0,1]$ and transits to the next state $s_{h+1}\in\mathcal{S}$ according to the probability $\mathbb{P}_h(\cdot|s_h, a_h)$.
The state-action distribution induced by $\pi$ on the $h$-th step is $d_{h}^\pi (s,a) := \text{Pr}(s_h=s, a_h=a|s_1\sim \mu; \pi)$.

The state value function for a policy $\pi$ is defined as $V_h^\pi(s;r) := \mathbb{E} \left[\sum_{t=h}^H r_t(s_t, a_t)|s_h=s, \pi \right]$.
The objective is to find a policy $\pi$ that maximizes the cumulative reward $J(\pi; r) := \mathbb{E}_{s_1\sim \mu}[V_1^\pi(s_1;r)]$.

The proof goes through in a similar way to that of \citet{jin2020reward}.

\begin{algorithm}[ht]
\caption{The planning algorithm}
\label{algorithm:planning}
\hspace*{0.02in} {\bf Input:} 
A dataset of transitions $\mathcal M$; The reward function $r$; The level of accuracy $\epsilon$.
\begin{algorithmic}[1]
\For{all $(s,a,s',h)\in \mathcal{S} \times \mathcal{A}\times\mathcal{S}\times [H]$}
        \State $N_h(s,a,s')\leftarrow \sum_{(s_h,a_h,s_{h+1})\in \mathcal M} \mathbb{I}[s_h=s,a_h=a,s_{h+1}=s']$
        \State  $N_h(s,a)\leftarrow \sum_{s'}N_h(s,a,s')$
        \State $\hat{\mathbb{P}}_h(s'|s,a)=N_h(s,a,s')/N_h(s,a)$
\EndFor
\State $\hat \pi \leftarrow \text{APPROXIMATE-MDP-SOLVER}(\mathbb{\hat P},r,\epsilon)$
\State \Return $\hat \pi$
\end{algorithmic}
\end{algorithm}

\begin{proof}[Proof of Theorem \ref{theorem:main}]

Given the dataset $\mathcal{M}$ specified in the theorem, a reward function $r$ and the level of accuracy $\epsilon$, we use the planning algorithm shown in Algorithm \ref{algorithm:planning} to obtain a near-optimal policy $\hat \pi$.
The planning algorithm first estimates a transition matrix $\hat{\mathbb{P}}$ based on $\mathcal{M}$ and then solves for a policy based on the estimated transition matrix $\hat{\mathbb{P}}$ and the reward function $r$.

Define $\hat{J}(\pi;r) := \mathbb{E}_{s_1 \sim \mu}[\hat{V}^\pi_1(s_1;r)]$ and $\hat{V}$ is the value function on the  MDP with the transition dynamics $\hat{\mathbb{P}}$. We first decompose the error into the following terms:
\begin{equation*}
\begin{aligned}
J(\pi^* ; r) - J(\hat{\pi};r)
\le &
\underbrace{|\hat{J}(\pi^*;r) - J(\pi^*;r)|}_\text{evaluation error 1}  + 
\underbrace{(\hat{J}(\pi^*; r) - \hat{J}(\hat{\pi}^*; r))}_{\le 0 \text{ by definition}} + 
\underbrace{(\hat{J}(\hat{\pi}^*; r) - \hat{J}(\hat{\pi}; r))}_\text{optimization error} + 
\underbrace{|\hat{J}(\hat{\pi}; r)-J(\hat{\pi}; r)|}_\text{evaluation error 2}
\end{aligned}
\end{equation*}


Then, the proof of the theorem goes as follows: 

To bound the two evaluation error terms, we first present Lemma \ref{lemma:entropy} to show that the policy which maximizes R\'enyi entropy is able to visit the state-action space reasonably uniformly, by leveraging 
the convexity of the feasible region in the state-action distribution space 
(Lemma \ref{lemma:regularityK}).
Then, with Lemma \ref{lemma:entropy} and the concentration inequality Lemma \ref{lemma:concentration}, we show that the two evaluation error terms can be bounded by $\epsilon$ for any policy and any reward function (Lemma \ref{lemma:epsilon_close}).

To bound the optimization error term, we use the natural policy gradient (NPG) algorithm as the APPROXIMATE-MDP-SOLVER in Algorithm \ref{algorithm:planning} to solve for a near-optimal policy on the estimated $\hat{\mathbb{P}}$ and the given reward function $r$. 
Finally, we apply the optimization bound for the NPG algorithm \cite{agarwal2019optimality}
to bound the optimization error term (Lemma \ref{lemma:MDP_solver}).

\end{proof}

\begin{definition}[$\mathcal{K}_h$]
$\mathcal{K}_h$ is defined as the feasible region in the state-action distribution space 
$\mathcal{K}_h := \{d_{h}^\pi: \pi \in \Pi\}\subset \Delta^n$, $\forall h\in[H]$, where $n=SA$ is the cardinality of the state-action space.
\end{definition}

Hazan et al. \cite{hazan2018provably} proved the convexity of such feasible regions in the infinite-horizon and discounted setting. 
For completeness, we provide a similar proof on the convexity of $\mathcal{K}_h$ in the episodic setting.

\begin{lemma}
\label{lemma:regularityK}[Convexity of $\mathcal{K}_h$]
$\mathcal{K}_h$ is convex.
Namely,
for any $\pi_1,\pi_2 \in \Pi$ and $h\in[H]$, denote $d_1=d_{h}^{\pi_1}$ and $d_2=d_{h}^{\pi_2}$ , there exists a policy $\pi\in \Pi$ such that $d:=\lambda d_1 + (1-\lambda) d_2 = d_{h}^\pi \in \mathcal{K}_h, \forall \lambda\in[0,1]$. 
\end{lemma}
\begin{proof}[Proof of Lemma \ref{lemma:regularityK}]
Define a mixture policy $\pi'$ 
that first chooses from $\{\pi_1, \pi_2\}$ with probability $\lambda$ and $1-\lambda$ respectively at the start of each episode, and then executes the chosen policy through the episode.
Define the state distribution (or the state-action distribution) for this mixture policy on each layer as $d_h^{\pi'}, \forall h\in [H]$.
Obviously, $d=d_h^{\pi'}$.
For any mixture policy $\pi'$, we can construct a policy $\pi:\mathcal{S}\times[H]\to\Delta^\mathcal{A}$ where $\pi(a|s,h) = \dfrac{d_h^{\pi'}(s,a)}{d_h^{\pi'}(s)}$ such that $d_{h}^{\pi}=d_{h}^{\pi'}$ \cite{puterman2014markov}.
In this way, we find a policy $\pi$ such that $d_{h}^{\pi}=d$.
\end{proof}

Similar to Jin et al. \cite{jin2020reward}, we define $\delta$-significance for state-action pairs on each step and show that the policy that maximizes R\'enyi entropy is able to reasonably uniformly visit the significant state-action pairs.

\begin{definition}[$\delta$-significance]
A state-action pair $(s,a)$ on the $h$-th step is $\delta$-significant if there exists a policy $\pi\in\Pi$, such that the probability to reach $(s,a)$ following the policy $\pi$ is greater than $\delta$, i.e., $\max_\pi d_{h}^\pi(s,a) \ge \delta$.
\end{definition}

Recall the way we construct the dataset $\mathcal{M}$: 
With the set of policies $\varpi:=\{\pi^{(h)}\}_{h=1}^H$ where $\pi^{(h)}: S\times [H] \to \Delta^\mathcal{A}$
and
$\pi^{(h)}\in arg \max_\pi H_\alpha(d_h^\pi), \forall h\in [H]$, we first uniformly randomly choose a policy $\pi$ from $\varpi$ at the start of each episode, and then execute the policy $\pi$ through the episode. 
Note that there is a set of policies that maximize the R\'enyi entropy on the $h$-th layer since the policy on the subsequent layers does not affect the entropy on the $h$-th layer.
We denote the induced state-action distribution on the $h$-th step as $\mu_h$, i.e.,
the dataset $\mathcal{M}$ can be regarded as being sampled from $\mu_h$ for each step $h\in[H]$.
Therefore, $\mu_h(s,a) = \frac{1}{H} \sum_{t=1}^H d_h^{\pi^{(t)}}(s,a), \forall s\in\mathcal{S}, a\in\mathcal{A}, h\in [H]$.

\begin{lemma}
\label{lemma:entropy}
If $0<\alpha<1$, then
\begin{equation}
\label{eq:entropy}
    \forall \delta\text{-significant } (s,a,h), \ 
\dfrac{\max_\pi d_{ h}^\pi(s,a)}{\mu_h(s,a)} \le 
\dfrac{SAH}{\delta^{\alpha / (1-\alpha)}}.
\end{equation}
\end{lemma}

\begin{proof}[Proof of Lemma \ref{lemma:entropy}]
For any $\delta$-significant $(s,a,h)$, consider the policy $\pi'\in arg\max_\pi d_{h}^\pi (s,a)$. 
Denote $x := d_{h}^{\pi'} \in \Delta(\mathcal{S}\times \mathcal{A})$.
We can treat $x$ as a vector with $n:=SA$ dimensions and use the first dimension to represent $(s,a)$, i.e., $x_1 = d_{h}^{\pi'}(s,a)$. 
Denote $z := d_{h}^{\pi^{(h)}}$.
Since $\pi^{(h)}\in arg \max_\pi H_\alpha(d_h^\pi)$, $z = arg \max_{d\in \mathcal{K}_h} H_\alpha(d)$.
By Lemma \ref{lemma:regularityK},  $y_\lambda := (1-\lambda) x + \lambda z \in \mathcal{K}_h, \forall \lambda \in [0, 1]$.
Since $H_\alpha(\cdot)$ is concave over $\Delta^n$ when $0<\alpha<1$, $H_\alpha(y_\lambda)$ will monotonically increase as we increase $\lambda$ from 0 to $1$, i.e.,
\[
\dfrac{\partial H_\alpha(y_\lambda)}{\partial \lambda} \ge 0, \quad \forall \lambda \in [0, 1].
\]
This is true when $\lambda = 1$, i.e., when $y_\lambda = z$, we have
\[
\begin{aligned}
\dfrac{\partial H_\alpha(y_\lambda)}{\partial \lambda} \Big|_{\lambda=1}
= & \frac {\alpha}{1-\alpha}\frac{\sum_{i=1}^n\left((1-\lambda)x_i+\lambda z_i\right)^{\alpha-1}(z_i-x_i)}{\sum_{i=1}^n\left((1-\lambda)x_i+\lambda z_i\right)^\alpha}\Big|_{\lambda=1}  \\
= & \frac {\alpha}{1-\alpha}\frac {\sum_{i=1}^n z_i^{\alpha-1}(z_i-x_i)}{\sum_{i=1}^nz_i^\alpha}
\propto \sum_{i=1}^n z_i^\alpha -\sum_{i=1}^n z_i^{\alpha-1}x_i\ge 0.\\
\end{aligned}
\]

Since $x_i, z_i$ are non-negative, we have
\[
\sum_{i=1}^n z_i^\alpha \ge \sum_{i=1}^n x_i z_i^{\alpha -1} = \sum_{i=1}^n \left(\dfrac{x_i}{z_i}\right)^{1-\alpha} x_i^\alpha \ge \left(\frac{x_1}{z_1}\right)^{1-\alpha}x_1^\alpha.
\]


Note that $x_1 \ge \delta > 0$, we have
\begin{equation}
    \dfrac{x_1}{z_1} 
    \le \left( \dfrac{\sum_{i=1}^n z_i^\alpha}{x_1^\alpha} \right)^{1/(1-\alpha)} 
    \le \left( \dfrac{n (\frac{1}{n})^\alpha}{\delta^\alpha} \right)^{1/(1-\alpha)} 
    = \dfrac{n}{\delta^{\alpha/(1-\alpha)}}.
\end{equation}
Since  $\mu_h(s,a) = \frac{1}{H} \sum_{t=1}^H d_h^{\pi^{(t)}}(s,a), \forall s\in\mathcal{S}, a\in\mathcal{A},h\in [H]$,
we have
$\mu_h(s,a) \ge \frac{1}{H}d_{h}^{\pi^{(h)}}(s,a)=\frac{1}{H} z_1$.
Using the fact that $x_1 = \max_\pi d_{h}^\pi(s,a)$, we can obtain the inequality in \eqref{eq:entropy}.
\end{proof}

When $\alpha \to 0^+$, we have
\[
    \forall \delta\text{-significant } (s,a,h), \ 
\dfrac{\max_\pi d_{ h}^\pi(s,a)}{\mu_h(s,a)} \le 
SAH.
\]
When $\alpha = \frac{1}{2}$,
\[
    \forall \delta\text{-significant } (s,a,h), \ 
\dfrac{\max_\pi d_{ h}^\pi(s,a)}{\mu_h(s,a)} \le 
SAH/\delta.
\]

\begin{lemma}
\label{lemma:concentration}
Suppose $\hat{\mathbb{P}}_h$ is the empirical transition matrix estimated from $M$ samples that are i.i.d. drawn from $\mu_h$ and $\mathbb{P}_h$ is the true transition matrix, then with probability at least $1-p$, for any $h\in[H]$ and any value function $G:S\to[0,H]$, we have:
\begin{equation}
    \mathbb{E}_{(s,a) \sim \mu_h} |(\hat{\mathbb{P}}_h - \mathbb{P}_h) G (s,a)|^2 \le O\left(\dfrac{H^2 S}{M} \log(\dfrac{HM}{p})\right).
\end{equation}
\end{lemma}

\begin{proof}[Proof of Lemma \ref{lemma:concentration}]
The lemma is proved in a similar way to that of  Lemma C.2 in \cite{jin2020reward} that uses a concentration inequality and a union bound. 
The proof differs in that we do not need to count the different events on the action space for the union bound
since the state-action distribution $\mu_h$ is given here (instead of only the state distribution).
This results in a missing factor $A$ in the logarithm term compared with their lemma.
\end{proof}

\begin{lemma}[Lemma E.15 in Dann et al. \cite{dann2017unifying}]
\label{lemma:dann}
For any two MDPs $M'$ and $M''$  with rewards $r'$ and $r''$ and transition probabilities $\mathbb P'$ and $\mathbb P''$, the difference in values $V'$ and $V''$ with respect to the same policy $\pi$ is
\[
V'_h(s)-V''_h(s)=\mathbb E_{M'',\pi}\left[\left.\sum_{i=h}^H[r'_i(s_i,a_i)-r''_i(s_i,a_i)+(\mathbb P'_i-\mathbb P''_i)V'_{i+1}(s_i,a_i)]\right |s_h=s\right].
\]
\end{lemma}

\begin{lemma}
\label{lemma:epsilon_close}
Under the conditions of Theorem \ref{theorem:main}, with probability at least $1-p$, for any reward function $r$ and any policy $\pi$, we have:
\begin{equation}
\label{eq:epsilon_close}
    |\hat{J}(\pi; r) - J(\pi; r)| \le \epsilon.
\end{equation}
\end{lemma}

\begin{proof}[Proof of Lemma \ref{lemma:epsilon_close}]
The proof is similar to Lemma 3.6 in Jin et al. \cite{jin2020reward} but differs in several details.
Therefore, we still provide the proof for completeness.
The proof is based on the following observations: 
1) The total contribution of all insignificant state-action pairs is small;
2) $\hat{\mathbb{P}}$ is reasonably accurate for significant state-action pairs due to the concentration inequality in Lemma \ref{lemma:concentration}.

Using Lemma \ref{lemma:dann} on $M$ and $\hat M$ with the true transition dynamics $\mathbb P$ and the estimated transition dynamics $\hat {\mathbb P}$ respectively, we have
\[
\begin{aligned}
|\hat{J}(\pi; r) - J(\pi; r)|=
\left|\mathbb E_{s_1\sim \mathbb P_1}
\left[\hat V_1^\pi(s_1;r) -V_1^\pi(s_1;r)
\right]
\right|
& \le
\left|
\mathbb E_\pi\sum_{h=1}^H(\hat{\mathbb{P}}_h-\mathbb P_h)\hat V_{h+1}^\pi(s_h,a_h)
\right| \\
& \le
\mathbb E_\pi\sum_{h=1}^H
\left|
(\hat{\mathbb{P}}_h-\mathbb P_h)\hat V_{h+1}^\pi(s_h,a_h)
\right|.
\end{aligned}
\]
Let $\mathcal S_h^\delta=\{(s,a):\max_\pi d_{h}^\pi(s,a)\ge\delta\}$ be the set of $\delta$ -significant state-action pairs on the $h$-th step.
We have
\[
\begin{aligned}
\mathbb E_\pi & |(\hat{\mathbb{P}}_h-\mathbb P)\hat V_{h+1}^\pi(s_h,a_h)|
\le \underbrace{\sum_{(s,a)\in\mathcal S_h^\delta}|(\hat{\mathbb{P}}_h-\mathbb P_h)\hat V^\pi_{h+1}(s,a)|d_{h}^\pi(s,a)}_{\xi_h}+\underbrace{\sum_{(s,a)\notin\mathcal S_h^\delta }|(\hat{\mathbb{P}}_h-\mathbb P_h)\hat V^\pi_{h+1}(s,a)|d_{h}^\pi(s,a)}_{\zeta_h}.
\end{aligned}
\]
By the definition of $\delta$-significant $(s,a)$ pairs, we have
\[
\zeta_h\le H\sum_{(s,a)\notin \mathcal{S}_h^\delta}d_{h}^\pi(s,a)\le H\sum_{(s,a)\notin\mathcal S_h^\delta}\delta \le SAH\delta.
\]
As for $\xi_h$, by Cauchy-Shwartz inequality and Lemma \ref{lemma:entropy}, we have
\begin{align*}
\xi_h  &\le\left[\sum_{(s,a)\in\mathcal S_h^\delta}|(\hat{\mathbb{P}}_h-\mathbb P_h)\hat V_{h+1}^\pi(s,a)|^2d_{h}^\pi(s,a)\right]^\frac 1 2\\
&\le\left[\sum_{(s,a)\in\mathcal S_h^\delta}|(\hat{\mathbb{P}}_h-\mathbb P_h)\hat V_{h+1}^\pi(s,a)|^2 \frac{SAH}{\delta^{\alpha/(1-\alpha)}}\mu_h(s,a)\right]^\frac 1 2\\
& \le\left[\frac{SAH}{\delta^{\alpha/(1-\alpha)}}\mathbb{E}_{\mu_h}|(\hat{\mathbb{P}}_h-\mathbb P_h)\hat V_{h+1}^\pi(s,a)|^2 \right]^\frac 1 2.
\end{align*}
By Lemma \ref{lemma:concentration}, we have
\[
\mathbb E_{(s,a)\sim \mu_h}|(\hat{\mathbb{P}}_h-\mathbb P)\hat V_{h+1}^\pi(s,a)|^2\le O\left(\frac{H^2S}{M}\log(\frac{HM}{p})\right).
\]
Thus, we have
\[
\mathbb E_\pi|(\hat{\mathbb{P}}_h-\mathbb P)\hat V_{h+1}^\pi(s_h,a_h)|\le\xi_h+\zeta_h\le O\left(\sqrt{\frac{H^3S^2A}{M\delta^{\alpha/(1-\alpha)}}\log(\frac{HM}{p})}\right)+HSA\delta.
\]
Therefore, combining inequalities above, we have
\[
\begin{aligned}
|\mathbb E_{s_1\sim \mathbb P_1}
[\hat V_1^\pi(s_1;r) -V_1^\pi(s_1;r)]|
& \le\mathbb E_\pi\sum_{h=1}^H|(\hat{\mathbb{P}}_h-\mathbb P_h)\hat V_{h+1}^\pi(s_h,a_h) | \\
& \le O\left(\sqrt{\frac{H^5S^2A}{M\delta^{\alpha/(1-\alpha)}}\log(\frac{HM}{p})}\right)+H^2SA\delta \\
& \le O\left(H^2SA\left(\frac{\sqrt{B}}{\delta^{\beta}}+\delta\right)\right),
\end{aligned}
\]
where $\beta =\frac {\alpha}{2(1-\alpha)}$ and $B=\frac{H}{AM}\log(\frac {HM}{p})$. 
With $\delta = (\sqrt{B})^{\frac{1}{\beta+1}}$, it is sufficient to ensure that $|\hat{J}(\pi;r)-J(\pi;r)| \le \epsilon$, if we set

\[
M\ge O\left(\left(\frac{H^2SA}{\epsilon}\right)^{2(\beta+1)}\frac{H}{A}\log\left(\frac{SAH}{\epsilon p}\right)\right).
\]

Especially, for $\alpha \to 0^+$, we have
\[
|\hat{J}(\pi;r) - J(\pi;r)| \le O \left( \sqrt{\dfrac{H^5S^2A}{M} \log (\dfrac{HM}{p})} \right) + H^2 S A \delta.
\]
To establish \eqref{eq:epsilon_close}, we can choose $\delta=\epsilon/2SAH^2$ and set $M \ge c \dfrac{H^5S^2A}{\epsilon^2}\log \left( \dfrac{SAH}{p\epsilon} \right)$ for sufficiently large absolute constant $c$.
\end{proof}

We use the natural policy gradient (NPG) algorithm as the approximate MDP solver.
\begin{lemma}[\cite{agarwal2019optimality,jin2020reward}]
\label{lemma:MDP_solver}
With learning rate $\eta$ and iteration number $T$, the output policy $\pi^{(T)}$ of the NPG algorithm satisfies the following:
\begin{equation}
    J(\pi^*;r) - J(\pi^{(T)};r) \le \dfrac{H \log A}{\eta T} + \eta H^2
\end{equation}
By choosing $\eta = \sqrt{\log A / HT}$ and $T = 4H^3 \log A / \epsilon^2$, the policy $\pi^{(T)}$ returned by NPG is $\epsilon$-optimal.
\end{lemma}

\section{Policy gradient of state-action space entropy}
\label{app:PG}


In the following derivation, we use the notation that sums over the states or the state-action pairs. 
Nevertheless, it can be easily extended to continuous state-action spaces by simply replacing summation by integration.
Note that, although the R\'enyi entropy for continuous distributions is not the limit of the R\'enyi entropy for discrete distributions (since they differs by a constant \cite{sanei2016renyi}), their gradients share the same form.
Therefore, the derivation still holds for the continuous case.

\begin{lemma}[Gradient of discounted state visitation distribution]
\label{lemma:distribution_gradient}
A policy $\pi$ is parameterized by $\theta$. The gradient of the discounted state visitation distribution induced by the policy w.r.t. the parameter $\theta$ is 
\begin{equation}
    \nabla_\theta d_\mu^\pi (s') = \frac{1}{1-\gamma} \mathbb{E}_{(s, a)\sim d_\mu^\pi} [\nabla_\theta \log \pi(a|s) d_{s,a}^\pi(s')]
\end{equation}
\end{lemma}

\begin{proof}
The discounted state visitation distribution started from one state-action pair can be written as the distribution started from the next state-action pair.

$$
d_{s_0, a_0}^\pi (s') = (1-\gamma) \mathbb{I}(s_0 =s') + \gamma \sum_{s_1, a_1} P^\pi(s_1,a_1|s_0,a_0) d_{s_1, a_1}^\pi (s'),
$$
where $P^\pi(s',a'|s,a): = \mathbb{P}(s'|s,a)\pi(a'|s')$.

Then, the proof follows an unrolling procedure on the discounted state visitation distribution as follows:

\begin{equation*}
\begin{aligned}
\nabla_\theta d^\pi_{s_0, a_0}(s') & = \sum_{s_1,a_1}\gamma \nabla_\theta \left[ P(s_1|s_0,a_0) \pi(a_1|s_1) d_{s_1, a_1}^\pi(s') \right] \\
& =\sum_{s_1,a_1} \gamma P(s_1|s_0,a_0) \left[\nabla_\theta \pi(a_1|s_1) d_{s_1, a_1}^\pi(s') + \pi(a_1|s_1) \nabla_\theta d_{s_1, a_1}^\pi(s') \right] \\
& =\sum_{s_1,a_1} \gamma P^\pi(s_1,a_1|s_0,a_0) \left[\nabla_\theta \log \pi(a_1|s_1) d_{s_1, a_1}^\pi(s') + \nabla_\theta d_{s_1, a_1}^\pi(s') \right] \quad \text{(unroll)} \\
& = \frac{1}{1-\gamma} \sum_{s,a} \sum_{t=1}^\infty \gamma^t P(s_t=s, a_t=a|s_0, a_0; \pi) [\nabla_\theta \log \pi(a|s) d_{s,a}^\pi (s')] \\
& = \frac{1}{1-\gamma} \mathbb{E}_{(s,a)\sim d_{s_0,a_0}^\pi} [\nabla_\theta \log \pi(a|s) d_{s,a}^\pi (s')] - \nabla_\theta \log \pi(a_0|s_0) d_{s_0,a_0}^\pi (s').
\end{aligned}
\end{equation*}

At last,

\begin{equation*}
    \nabla_\theta d_\mu^\pi(s') = \mathbb{E}_{s\sim\mu, a_0 \sim \pi(\cdot|s_0)} [\nabla_\theta \log \pi(a_0|s_0) d_{s_0,a_0}^\pi (s') + \nabla_\theta d_{s_0,a_0}^\pi (s')] = \frac{1}{1-\gamma} \mathbb{E}_{(s,a)\sim d_\mu^\pi} [\nabla_\theta \log \pi(a|s) d_{s,a}^\pi (s')].
\end{equation*}
\end{proof}


\begin{lemma}[Policy gradient of the Shannon entropy over the state-action space]
\label{lemma:pg_shannon}
Given a policy $\pi$ parameterized by $\theta$, the gradient of $H_1(d_\mu^\pi)$ w.r.t. the parameter $\theta$ is as follows:
\begin{equation}
\label{eq:pg_shannon_lemma}
    \nabla_\theta H_1(d_\mu^\pi) = \mathbb{E}_{(s,a)\sim d_\mu^\pi} \left[ \nabla_\theta \log \pi(a|s) \left( \frac{1}{1-\gamma} \langle d_{s,a}^\pi, -\log d_\mu^\pi  \rangle -\log d_\mu^\pi(s,a) \right) \right].
\end{equation}
\end{lemma}

\begin{proof}
First, we can calculate the gradient of Shannon entropy $\nabla_d H(d) = (-\log d - c)$ where $d$ is the state-action distribution and $c = - \sum_{s,a} \log d(s,a)$ is a normalization factor.
Then, we observe that 
$
\nabla_\theta d_\mu^\pi(s,a) = \nabla_\theta [d_\mu^\pi(s) \pi(a|s)]
= \nabla_\theta d_\mu^\pi(s) \pi(a|s) + d_\mu^\pi(s, a) \nabla_\theta \log \pi(a|s)
$.

Using the standard chain rule and Lemma \ref{lemma:distribution_gradient}, we can obtain the policy gradient:
\begin{equation*}
\begin{aligned}
\nabla_\theta H_1(d_\mu^\pi) = & \sum_{s, a} \nabla_\theta d_\mu^\pi(s,a) 
\dfrac{\partial H_1(d_\mu^\pi)}{\partial d_\mu^\pi(s,a)}
= \sum_{s, a} \nabla_\theta d_\mu^\pi(s,a) (-\log d_\mu^\pi(s,a) - c ) 
= - \sum_{s, a} \nabla_\theta d_\mu^\pi(s,a) \log d_\mu^\pi(s,a)\\ 
= & - \frac{1}{1-\gamma} \sum_{s',a'} d_\mu^\pi (s',a') \nabla_\theta \log \pi (a'|s') \sum_{s,a} d_{s',a'}^\pi (s,a) \log d_\mu^\pi(s,a) -  \sum_{s,a} d_\mu^\pi(s,a) \nabla_\theta \log \pi(a|s) \log d_\mu^\pi(s,a) \\
= & \mathbb{E}_{(s,a)\sim d_\mu^\pi} \left[ \nabla_\theta \log \pi(a|s) \left( \frac{1}{1-\gamma} \langle d_{s,a}^\pi, -\log d_\mu^\pi  \rangle -\log d_\mu^\pi(s,a) \right) \right], \\
\end{aligned}
\end{equation*}
where $\langle d_{s,a}^\pi, -\log d_\mu^\pi\rangle$ is the cross-entropy between the two distributions.
We use the fact $\mathbb{E}_{x\sim p} \left[c \nabla_\theta \log p(x) \right] = 0$ in the first line.
\end{proof}




\begin{lemma}[Policy gradient of the R\'enyi entropy over the state-action space]
\label{lemma:pg_renyi}
Given a policy $\pi$ parameterized by $\theta$ and $0<\alpha<1$, the gradient of $H_\alpha(d_\mu^\pi)$ w.r.t. the parameter $\theta$ is as follows:
\begin{equation}
\label{eq:pg_renyi_lemma}
    \nabla_\theta H_\alpha(d_\mu^\pi) \propto 
    \frac \alpha {1-\alpha} \mathbb E_{(s,a)\sim d_\mu^\pi}\left[\nabla_\theta\log\pi(a|s)\left(
    \frac{1}{1-\gamma}
    \langle d_{s,a}^\pi, (d_\mu^\pi)^{\alpha-1}
    \rangle
    +\left(d_\mu^\pi(s,a)\right)^{\alpha-1}\right)\right].
\end{equation}
\end{lemma}





\begin{proof}
First, we note that the gradient of R\'enyi entropy is
\[
\nabla_d H_\alpha(d) = \dfrac{\alpha}{1-\alpha} \dfrac{d^{\alpha-1}-c}{\sum_{i=1}^n d_i^\alpha},
\]
where $c$ is a constant normalizer. Given a vector $d=(d_1, \cdots, d_n)^T$ and a real number $a$, we use $d^a$ to denote the vector $(d_1^a, \cdots, d_n^a)^T$.
Similarly, using the chain rule and Lemma \ref{lemma:distribution_gradient}, we have
\begin{equation*}
\begin{aligned}
\nabla_\theta H_\alpha(d_\mu^\pi) &=\sum_{s, a} \nabla_\theta d_\mu^\pi(s,a) 
\dfrac{\partial H_\alpha(d_\mu^\pi)}{\partial d_\mu^\pi(s,a)}
\propto\frac \alpha {1-\alpha}  \sum_{s, a} \nabla_\theta d_\mu^\pi(s,a) \left[ (d_\mu^\pi(s,a))^{\alpha-1} - c\right] \\
&= \frac \alpha {1-\alpha}  \sum_{s, a} \nabla_\theta d_\mu^\pi(s,a) (d_\mu^\pi(s,a))^{\alpha-1}
\\
&= \frac \alpha {1-\alpha} \bigg( \frac{1}{1-\gamma} \sum_{s',a'}d_\mu^\pi(s',a')\nabla_\theta\log\pi(a'|s')
\sum_{s,a}
d_{s',a'}^\pi(s,a) (d_\mu^\pi(s,a))^{\alpha-1} \\
& \quad \quad + \sum_{s,a} d_\mu^\pi(s,a)\nabla_\theta\log\pi(a|s) (d_\mu^\pi(s,a))^{\alpha-1}
\bigg) \\
&=  \frac \alpha {1-\alpha} \mathbb E_{s,a\sim d_\mu^\pi}\left[\nabla_\theta\log\pi(a|s)\left(
\frac{1}{1-\gamma}
\langle d_{s,a}^\pi, (d_\mu^\pi)^{\alpha-1}
\rangle
+\left(d_\mu^\pi(s,a)\right)^{\alpha-1}\right)\right].
\end{aligned}
\end{equation*}

\end{proof}

We should note that when $\alpha \to 0$, $\nabla_d H_\alpha(d) \to 0$; and when $\alpha \to 1$, $\nabla_d H_\alpha(d) \to -\log d$.
Especially, we can write down the case for $\alpha = \frac{1}{2}$:
\begin{equation}
    \nabla_\theta H_\alpha(d_\mu^\pi) = \mathbb{E}_{s,a\sim d_\mu^\pi} \left[ \nabla_\theta \log \pi(a|s) 
    \left( 
    \frac{1}{1-\gamma} \langle d_{s,a}^\pi, \sqrt{\dfrac{1}{d_\mu^\pi}} \rangle + \sqrt{\dfrac{1}{d_\mu^\pi (s,a)}} \right) \right]
\end{equation}

This leads to a policy gradient update with reward $\propto (d_\mu^\pi)^{-1/2}$. This inverse-square-root style reward is frequently used as the exploration bonus in standard RL problems \cite{strehl2008analysis,bellemare2016unifying} and similar to the UCB algorithm \cite{auer2002finite,lai1985asymptotically} in the bandit problem.

\section{Practical implementation}
\label{app:explain}


Recall the deduced policy gradient formulations of the entropies on a CMP in \eqref{eq:pg_shannon_lemma} and \eqref{eq:pg_renyi_lemma}.
We observe that if we set the reward function $r(s,a) = (d_\mu^\pi(s,a))^{\alpha-1}$ (when $0<\alpha<1$) or $r(s,a) = -\log d_\mu^\pi(s,a)$ (when $\alpha\to 1$), these policy gradient formulations can be compared with the policy gradient for a standard MDP: The first term resembles the policy gradient of cumulative reward $\mathbb{E}_{(s,a)\sim d_\mu^\pi}[\nabla_\theta \log \pi(a|s) Q^\pi(s,a)]$ and the second term resembles the policy gradient of instant reward $\mathbb{E}_{(s,a)\sim d_\mu^\pi}[\nabla_\theta \log \pi(a|s) r(s,a)]$.


For the first term, a wide range of existing policy gradient based algorithms can be used to estimate this term. 
Since these algorithms such as PPO do not retain the scale of the policy gradient, we induce a coefficient $\eta$ to adjust for the scale between this term and the second term.

For the second term, when $\alpha=1$,
this term can be reduced to
\[
\mathbb{E}_{(s,a)\sim d_\mu^\pi}[\nabla_\theta \log \pi(a|s) (-\log d_\mu^\pi(s,a))] = \nabla_\theta \mathbb{E}_{s\sim d_\mu^\pi} [H_1(\pi(\cdot|s))],
\]
where $H_1(\pi(\cdot|s))$ denotes the Shannon entropy over the action distribution conditioned on the state sample $s$. 
Although the two sides are equivalent in expectation, the right hand side leads to smaller variance when estimated using samples due to the following reasons: 
First, the left hand side is calculated with the values of a density estimator (that estimates $d_\mu^\pi$) on the state-action samples and the density estimator suffers from high variance. 
In contrast, the right hand side does not involve the density estimation. 
Second, to estimate the expectation with samples, the left hand side uses both the state samples and the action samples while the right hand side uses only the state samples and the action distributions conditioned on these states.
This reduces the variance in the estimation.
When $0<\alpha<1$, 
since the second term in equation 
\eqref{eq:PG_alpha}
can hardly be reduced to such a simple form, we still use $\nabla_\theta \mathbb{E}_{s\sim d_\mu^\pi} [H_1(\pi(\cdot|s))]$ as the approximation.



\section{Value function estimator}
\label{app:value}


In this section, we provide the detail of the truncated TD($\lambda$) estimator \cite{sutton2018introduction} that we use in our algorithm.
Recall that in the $i$-th iteration, we maintain the replay buffer $\mathcal{D} = \mathcal{D}_i \cup \mathcal{D}_{i-1} \cup \cdots \cup \mathcal{D}_{i-n+1}$ as $\mathcal{D}=\{(s_{j1},a_{j1}), \cdots, (s_{jN_j}, a_{jN_j})\}_{j=1}^{nm}$ where $N_j$ is the length of the $j$-th trajectory and $nm$ is the total number of trajectories in the replay buffer.
Correspondingly, the reward calculated using the latest density estimator $P_\phi$ is denoted as $r_{jk} = (P_\phi(s_{jk},a_{jk}))^{\alpha-1}$ or $r_{jk} = -\log P_\phi(s_{jk},a_{jk})$ for each $(s_{jk},a_{jk}) \in \mathcal{D}$ 
, $\forall j\in[nm], k\in[N_j]$. 

The truncated n-step value function estimator can be written as:
\[
V^\text{nstep}_p(s_{jk}) := \sum_{\tau=k}^{h-1}
r_{j\tau} + 
V_\psi(s_{jh}) \text{,   with } h=\min(N_j, k+p).
\]
where $V_\psi$ is the current value function network with the parameter $\psi$.
Notice that, since the discount rate is considered when we collect sample, here we estimate the undiscounted sum.
When $p\ge N_j-k$ (i.e., $h=N_j$), this corresponds to the Monte Carlo estimator whose variance is high. 
When $p=1$ (i.e., $h=k+1$), this corresponds to the TD(0) estimator that has large bias.
To balance variance and bias, we use the truncated TD($\lambda$) estimator which is the weighted sum of the truncated n-step estimators with different $p$, i.e.,
\[
V^\text{targ}_\lambda (s_{jk}) := (1-\lambda) \sum_{p=1}^{P-1} \lambda^{p-1} V^\text{nstep}_p(s_{jk}) + \lambda^{P-1} V^\text{nstep}_P(s_{jk}),
\]
where $\lambda \in [0,1]$ and $P\in \mathbb{N}^+$ are the hyperparameters.

\section{Experiment settings, hyperparameters and more experiments.}
\label{app:exp}

\subsection{Experiment settings and hyperparameters}

Throughout the experiments, we set the hyperparameters as follows: The coefficient that balances the two terms in the policy gradient formulation is $\eta=10^{-4}$. 
The parameter for computing GAE and the target state values is $\lambda=0.95$.
The steps for computing the target state values is $P=15$.
The learning rate for the policy network is $4 \times 10^{-3}$ and the learning rate for VAE and the value function $10^{-3}$.

In \texttt{Pendulum} and \texttt{FourRooms} (see Appendix \ref{app:convergence}), 
we use $\gamma=0.995$ and store samples of the latest 10 iterations in the replay buffer.
To estimate the corresponding R\'enyi entropy of the policy during the training in the top of Figure \ref{fig:pendulum_fourrooms}, we sample several trajectories in each iteration. 
In each trajectory, we terminate the trajectory on each step with probability $1-\gamma$.
For $\alpha=0.5$ and $1$, we sample 1000 trajectories to estimate the distribution and then calculate the R\'enyi entropy over the state-action space.
For $\alpha=0$, we sample 50 groups of trajectories. 
In the $i$-th group, we sample 20 trajectories and count the number of unique state-action pairs visited $C_i$.
We use $\frac{1}{50} \sum_{i=1}^{50} \log C_i$ as the estimation of the R\'enyi entropy when $\alpha=0$.
The corresponding entropy of the optimal policy is estimated in a same manner.
The bottom figures in Figure \ref{fig:pendulum_fourrooms} is plotted using 100 trajectories, each of which is also collected in a same manner. 

In \texttt{MultiRooms}, 
we use $\gamma=0.995$ and store samples of the latest 10 iterations in the replay buffer.
We use the \texttt{MiniGrid-MultiRoom-N6-v0} task from minigrid. 
We fix the room configuration for a single run of the algorithm and use different room configurations in different runs. 
The four actions are turning left/right, moving forward and opening the door.
In the planning phase, we randomly pick $20$ grids and construct the corresponding reward functions. 
The reward function gives the agent a reward $r=1$ if the agent reaches the corresponding state and gives a reward $r=0$ otherwise.
We show the means and the standard deviations of five runs with different random seeds in the figures.

For Atari games, 
we use $\gamma=0.995$ and store samples of the latest 20 iterations in the replay buffer.
For Mujoco tasks, we use $\gamma=0.995$ and only store samples of the current iteration.
In Figure \ref{fig:montezuma}b, the two reward functions in the planning phase are designed as follows:
Assume the rooms are numbered in the default order by the environment (cf. the labels in Figure \ref{fig:montezuma}b). 
Under the first reward function, when the agent goes from Room 3 to Room 9 (the room below Room 3) for the first time, we assign a unit reward to the agent. Otherwise, we assign zero rewards to the agent.
Similarly, under the second reward function, we give a reward to the agent when it goes from Room 7 to Room 13 (the room below Room 7).

We will release the source code for this experiments once the paper is published.

\subsection{Convergence of MaxRenyi}
\label{app:convergence}
Convergence is one of the most critical properties of an algorithm. 
The algorithm based on maximizing Shannon entropy has a convergence guarantee leveraging the existing results on the convergence of standard RL algorithms and its similar form to the standard RL problem (noting that $H_1(d_\mu^\pi) = \langle d_\mu^\pi, -\log d_\mu^\pi \rangle$ and $J(\pi) = \frac{1}{1-\gamma} \langle d_\mu^\pi, r \rangle$) \cite{hazan2018provably}.
On the other hand, MaxRenyi does not have such a guarantee not only due to the form of R\'enyi entropy but also due to the techniques we use.
Therefore, we present a set of experiments here to illustrate \emph{whether our exploration algorithm MaxRenyi empirically converges to near-optimal policies in terms of the entropy over the state-action space}.
To answer this question, we compare the entropy induced by the agent during the training in MaxRenyi with the maximum entropy the agent can achieve.
We implement MaxRenyi for two simple environments, {Pendulum} and {FourRooms}, 
where we can solve for the maximum entropy by brute force search.
{Pendulum} from OpenAI Gym \cite{1606.01540} has a continuous state-action space.
For this task, the entropy is estimated by discretizing the state space into $16\times 16$ grids and the action space into two discrete actions.
{FourRooms} is a $11\times 11$ grid world environment with four actions and deterministic transitions. 
We show the results in Figure \ref{fig:pendulum_fourrooms}. 
The results indicate that our exploration algorithm approaches the optimal in terms of the corresponding state-action space entropy with different $\alpha$ in both discrete and continuous settings.

\begin{figure*}[t]
    \centering
    \begin{subfigure}{0.45\textwidth}
    \centering
    \includegraphics[width=\textwidth]{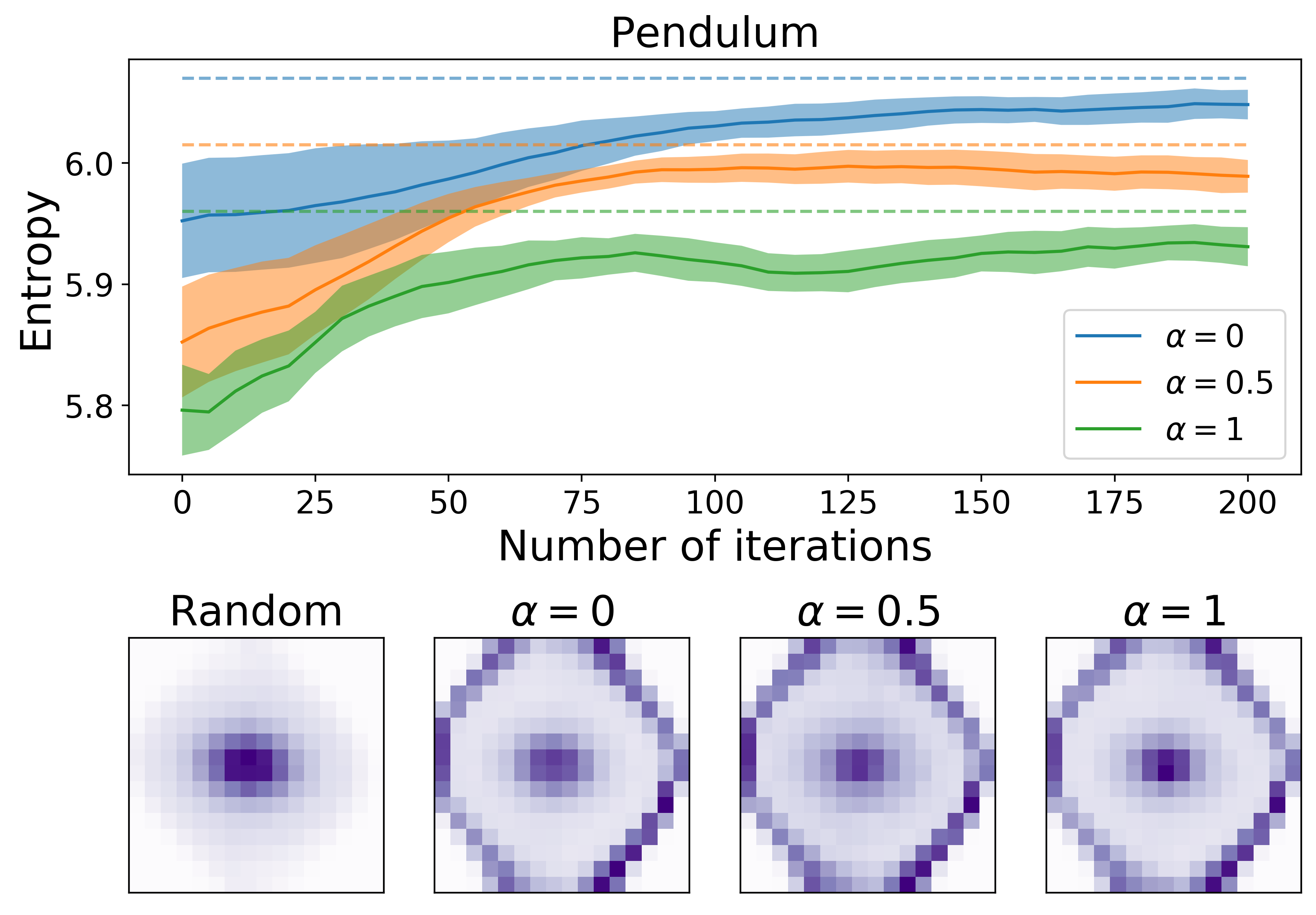}
    \end{subfigure}
    \hfill
    \begin{subfigure}{0.45\textwidth}
    \centering
    \includegraphics[width=\textwidth]{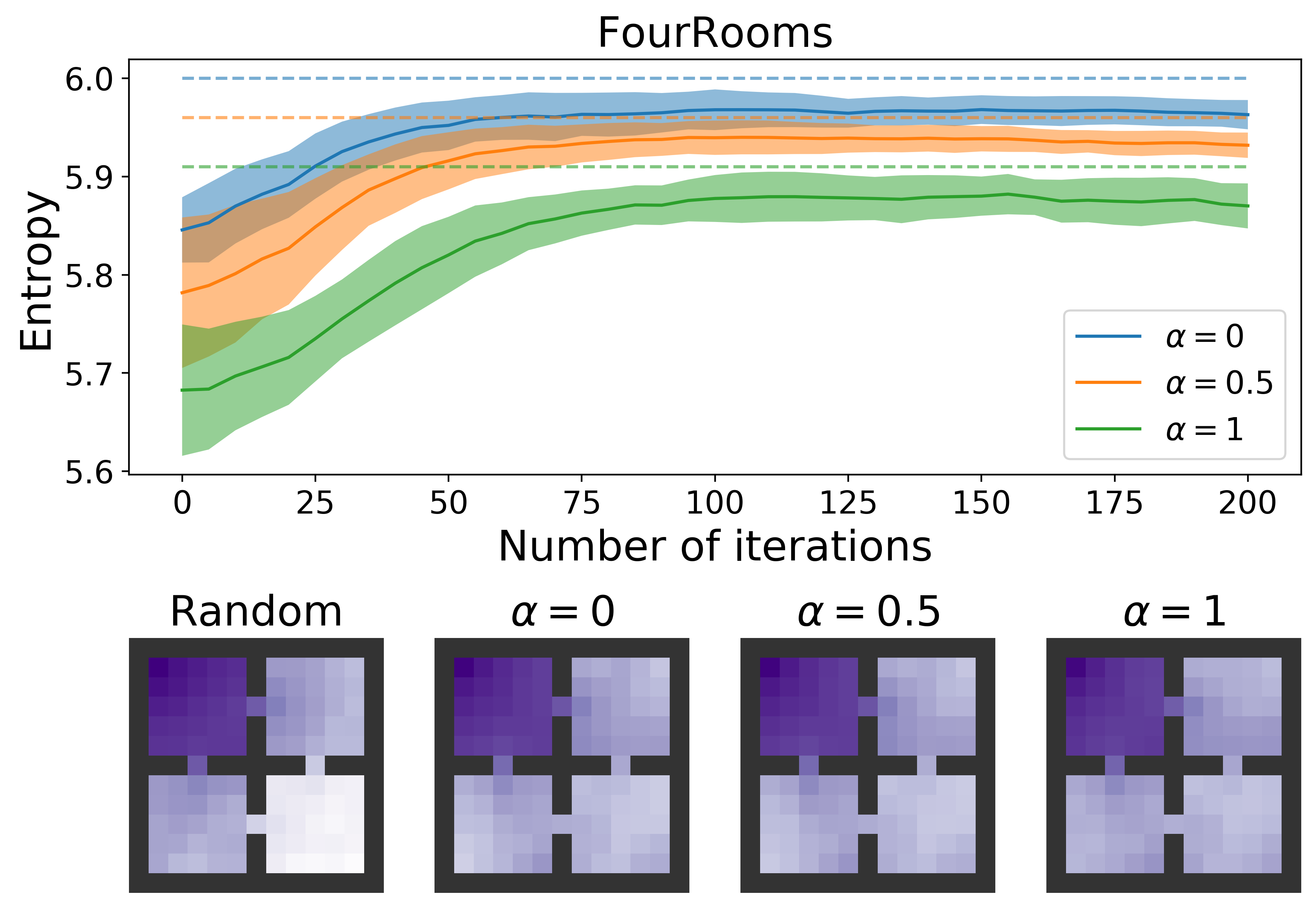}
    \end{subfigure}
    \caption{
    Top. The corresponding R\'enyi entropy of the state-action distribution induced by the policies during the training using our exploration algorithm with different $\alpha$.
    The dashed lines indicate the corresponding maximum entropies the agent can achieve.
    Bottom. The discounted state distributions of the learned policies compared with the random policy.}
    \label{fig:pendulum_fourrooms}
\end{figure*}

\subsection{More results for Atari games and Mujoco tasks}

The reward-free RL framework contains two phases, the exploration phase and the planning phase.
Due to space limit, we only show the final performance of the algorithms (i.e., the performance of the agent planned based on the data collected using the learned exploratory policy) in the main text. 
Here, we also show the performance of the exploration algorithms in the exploration phase (i.e., the performance for the reward-free exploration setting).
We show the results in Table \ref{table:3}.
First, the scores for the exploration phase are measured using the extrinsic game rewards of the tasks. 
However, the exploration algorithms are trained in the reward-free exploration setting (i.e., do not receive such extrinsic game rewards).
The results indicate that, although MaxRenyi is designed for the reward-free RL framework, it outperforms other exploration algorithms in the reward-free exploration setting.
Then, 
the average scores for the planning phase is the same as those in Table \ref{table:2}. In addition, we present the standard deviations over five random seeds. 
The results indicate that the stability of MaxRenyi is comparable to RND and ICM while achieving better performance on average.

\begin{table}[h]
\centering
\begin{tabular}{rrrrrr} 
& & MaxRenyi & RND & ICM & Random \\
\hline
\multirow{7}{*}{Exploration phase} & Montezuma & 
\textbf{7,540(2210)} & 4,220(736) & 1,140(1232) & 0 \\
& Venture & 
460(273) & \textbf{740(546)} & 120(147) & 0 \\
& Gravitar & 
\textbf{1,420(627)} & 1,010(886) & 580(621) & 173 \\
& PrivateEye & 
\textbf{6,480(7,813)} & 6,220(7,505) & 100(0) & 25 \\
& Solaris & 
\textbf{1,284(961)} & 1,276(712) & 1,136(928) & 1,236 \\
& Hopper & 
\textbf{207(628)} & 72(92) & 86(75) & 31 \\
& HalfCheetah & 
\textbf{456(328)} & 137(126) & 184(294) & -82 \\
\hline
\multirow{7}{*}{Planning phase} & Montezuma & 
\textbf{8,100(1,124)} & 4,740(2,216) & 380(146) & 0(0) \\
& Venture & 
\textbf{900(228)} & 840(307) & 240(80) & 0(0) \\
& Gravitar & 
1,850(579) & \textbf{1,890(665)} & 940(281) & 340(177) \\
& PrivateEye & 
\textbf{6,820(7,058)} & 3,040(5,880) & 100(0) & 100(0) \\
& Solaris & 
\textbf{2,192(1,547)} & 1,416(734) & 1,844(173) & 1,508(519) \\
& Hopper & 
\textbf{973(159)} & 819(87) &  704(59) & 78(36) \\
& HalfCheetah & 
\textbf{1,466(216)} & 1,083(377) & 700(488) & 54(50) \\
\end{tabular}
\caption{
The performance of the exploratory policies (\emph{Exploration phase}) and the planned policies under the reward-free RL framework (\emph{Planning phase}) corresponding to different exploratory algorithms. 
Random indicates using a random policy as the exploratory policy. 
We use $\alpha=0.5$ in MaxRenyi. 
The numbers before and in the parentheses indicate the means and the standard deviations over five random seeds respectively.}
\label{table:3}
\end{table}

\subsection{More results for Montezuma's Revenge}

Although the extrinsic game reward in Montezuma's Revenge aligns with the coverage of the state space, the number of rooms visited in one episode may be a more direct measure for an exploration algorithm. 
Therefore, we also provide the curves for the number of rooms visited during the reward-free exploration.
Compared with Figure \ref{fig:montezuma}a, we observe that, although MaxRenyi visits only slightly more rooms than RND, it yields higher returns by visiting more scoring states. 
This is possibly because our algorithm encourages the agent to visit the hard-to-reach states more.

\begin{figure*}[t]
    \centering
    \includegraphics[width=0.4\textwidth]{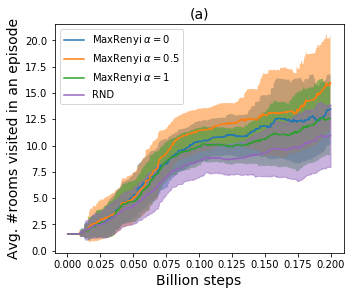}
    \caption{
    The average number of rooms visited in an episode along during the reward-free exploration.
    }
    \label{fig:montezuma_num_rooms}
\end{figure*}

Moreover, our algorithm is less likely to fail under different reward functions in the planning phase. There is a failure mode for the planning phase in Montezuma's Revenge: For example, with the reward to go across Room 3, the trained agent may first go to Room 7 to obtain an extra key and then return to open the door toward Room 0, which is not the shortest path. Out of 5 runs, our algorithms fail 2 times where RND fails all the 5 times. 

\end{alphasection}

\end{document}